\documentclass{article}
\usepackage{amssymb}
\usepackage{amsthm}
\usepackage{iclr2026_conference, times}

\usepackage{hyperref}       %
\usepackage{url}            %
\usepackage{amsmath,amsfonts,bm}
\usepackage{cleveref}
\usepackage{thm-restate}
\usepackage{bbm}
\usepackage{mathtools}
\usepackage{graphicx}

\usepackage[old]{old-arrows}
\usepackage{etoc}
\usepackage{setspace}
\usepackage{xspace}

\usepackage{color}

\definecolor{niceRed}{RGB}{190,38,38}
\definecolor{Red2}{RGB}{219, 50, 54}
\definecolor{mgreen}{RGB}{160, 200, 140}
\definecolor{blueGrotto}{RGB}{5,157,192}
\definecolor{limeGreen}{HTML}{81B622}
\definecolor{myellow}{rgb}{0.88,0.61,0.14}
\definecolor{darkGreen}{HTML}{2E8B57}
\definecolor{navyBlueP}{HTML}{03468F}
\definecolor{Sepia}{HTML}{7F462C}
\definecolor{red2}{HTML}{1F462C}
\definecolor{orange2}{HTML}{FF8000}
\definecolor{mgray}{HTML}{ABB3B8}
\definecolor{lgray}{HTML}{E5E8E9}
\definecolor{myPurple}{RGB}{175,0,124}
\definecolor{mypurple2}{rgb}{0.8,0.62,1}
\definecolor{royalBlue}{HTML}{057DCD}
\definecolor{mpink}{HTML}{FC6C85}
\definecolor{lblue}{RGB}{74,144,226}
\definecolor{peagreen}{RGB}{152,193,39}
\definecolor{typ_navy}{HTML}{001f3f}
\definecolor{typ_blue}{HTML}{0074d9}
\definecolor{typ_aqua}{HTML}{7fdbff}
\definecolor{typ_teal}{HTML}{39cccc}
\definecolor{typ_eastern}{HTML}{239dad}
\definecolor{typ_purple}{HTML}{b10dc9}
\definecolor{typ_fuchsia}{HTML}{f012be}
\definecolor{typ_maroon}{HTML}{85144b}
\definecolor{typ_red}{HTML}{ff4136}
\definecolor{typ_orange}{HTML}{ff851b}
\definecolor{typ_yellow}{HTML}{ffdc00}
\definecolor{typ_olive}{HTML}{3d9970}
\definecolor{typ_green}{HTML}{2ecc40}
\definecolor{typ_lime}{HTML}{01ff70}
\definecolor{newgreen}{HTML}{83c702}
\definecolor{newpurp}{RGB}{97,96,121}

\hypersetup{
	colorlinks = true,
	linkcolor = red,
	citecolor = royalBlue,
	linktocpage = true,
	urlcolor = royalBlue
}

\newtheorem{lemma}{Lemma}
\newtheorem{assumption}{Assumption}
\newtheorem{theorem}{Theorem}

\newtheorem{proposition}{Proposition}

\newcommand{\E}{\mathbb{E}}

\newcommand{\Reals}{\mathbb{R}}
\newcommand{\Naturals}{\mathbb{N}}
\newcommand{\Mcal}{\mathcal{M}}
\newcommand{\Xcal}{\mathcal{X}}
\newcommand{\Fcal}{\mathcal{F}}

\newcommand{\Ical}{\mathcal{I}}

\newcommand{\Ecal}{\mathcal{E}}
\newcommand{\cl}{\textup{cl}}

\newcommand{\quotes}[1]{``#1''}

\newcommand{\argmax}{\operatornamewithlimits{argmax}}

\newcommand{\Prob}{\mathbb{P}}

\newcommand{\bmMuHat}{\hat{\bm\mu}}
\newcommand{\bmMuTilde}{\tilde{\bm\mu}}
\newcommand{\muMin}{\mu_{\textup{min}}}
\newcommand{\muMax}{\mu_{\textup{max}}}

\makeatletter
\newcommand{\mapstoto}{\mathpalette\@mapstoto\relax}
\newcommand*{\@mapstoto}[2]{%
    \mathrel{%
      \vcenter{%
         \vbox{%
            \baselineskip\z@skip
            \lineskip\z@
            \ialign{##\cr$#1\mapstochar\varrightarrow$\cr
            $#1\mapstochar\varrightarrow$\cr}%
         }%
      }%
   }%
}
\makeatother

\newcommand{\ie}{\emph{i.e.},\xspace}
\newcommand{\eg}{\emph{e.g.},\xspace}

\crefname{assumption}{Assumption}{Assumptions}

\newcommand{\tas}{\textsc{TaS}\xspace}
\newcommand{\stas}{\textsc{S-TaS}\xspace}

\title{Non-Asymptotic Analysis of (Sticky) Track-and-Stop}

\author{Riccardo Poiani, Martino Bernasconi \& Andrea Celli\\
Bocconi University\\
Milan, Italy\\
\texttt{\{riccardo.poiani,martino.bernasconi,andrea.celli2\}@unibocconi.it} 
}

\iclrfinalcopy

\begin{document}
\maketitle
\begin{abstract}
    In pure exploration problems, a statistician sequentially collects information to answer a question about some stochastic and unknown environment. The probability of returning a wrong answer should not exceed a maximum risk parameter $\delta$ and good algorithms make as few queries to the environment as possible.
    The Track-and-Stop algorithm is a pioneering method to solve these problems. Specifically, it is well-known that it enjoys asymptotic optimality sample complexity guarantees for $\delta\to 0$ whenever the map from the environment to its correct answers is single-valued (\eg best-arm identification with a unique optimal arm).   
    The Sticky Track-and-Stop algorithm extends these results to settings where, for each environment, there might exist multiple correct answers (\eg $\epsilon$-optimal arm identification). 
    Although both methods are optimal in the asymptotic regime, their non-asymptotic guarantees remain unknown.
    In this work, we fill this gap and provide non-asymptotic guarantees for both algorithms. 
\end{abstract}

\section{Introduction}
In pure exploration problems, a statistician interacts with a set of $K \in \Naturals$ probability distributions denoted by $\bm\varphi = \left\{\varphi_i \right\}_{i\in[K]}$, commonly referred to as arms. Their unknown means are denoted by $\bm\mu = \left\{\mu_i \right\}_{i\in[K]}$, and $\bm\mu$ belongs to some set $\Mcal \subseteq \Reals^K$ which encodes some possibly known structure among the different arms, \eg Lipschitzianity \citep{wang2021fast} or unimodality \citep{poiani2024best}. During each step $t \in \Naturals$, the statistician chooses an arm $A_t$, and observes a sample $X_t \sim \varphi_{A_t}$ from the corresponding reward distribution. Given a maximum risk parameter $\delta \in (0,1)$, the statistician aims to answer a question about the unknown means $\bm\mu$ while using as few samples as possible. Specifically, there is a known answer space $\Ical$ and a (set-valued) answer function $i^{\star}(\bm\mu)$ that maps each bandit $\bm\mu$ to a subset of correct answers $i^{\star}(\bm\mu)$ within $\Ical$. The probability of returning an answer that does not belong to $i^{\star}(\bm\mu)$ should not exceed $\delta$. 

The most studied pure exploration problem is Best-Arm Identification \citep[BAI,][]{garivier2016optimal}, where the answer space is $\{1, \dots, K\}$ and the goal is to return the unique index of the arm with the highest mean, \ie $i^{\star}(\bm\mu) = \argmax_{k \in [K]} \mu_k$. The pioneering work by \citet{garivier2016optimal} developed a tight lower bound for the unstructured BAI problem and proposed the Track-and-Stop (\tas) algorithm to solve it. Remarkably, the expected number of samples required by \tas to identify $i^{\star}(\bm\mu)$ with high probability exactly matches the lower bound as $\delta$ approaches $0$. In this sense, \tas is asymptotically optimal for $\delta \to 0$. These results have been shown to hold even outside the BAI setting. Indeed, they extend to the more general structured partition identification problem \citep{kaufmann2021mixture}. Here, $\Mcal$ is partitioned into a finite number $|\Ical|$ of disjoint subsets, \ie $\Mcal = \bigcup_{i \in \Ical} \Mcal_i$, and the statistician aims to find $i$ such that $\bm\mu \in \Mcal_i$. In this sense, \tas turned out to be a powerful tool that can be used to solve a wide variety of problems (\ie all the problems where $i^{\star}(\bm\mu)$ is unique) while enjoying asymptotic optimality. If there exists multiple correct answers for a certain bandit $\bm\mu$, \ie $i^{\star}(\bm\mu)$ is multi-valued, \tas fails to achieve optimality \citep{degenne2019pure}. This is the case, for instance, of the $\epsilon$-best arm identification problem, where multiple arms $j$ might satisfy $\mu_j \ge \argmax_{k \in [k]} \mu_k - \epsilon$. To solve this issue, \citet{degenne2019pure} proposed a modification of the \tas algorithm called Sticky Track-and-Stop (\stas). \stas enjoys asymptotic optimality for any pure exploration problem with multiple correct answers.

Due to their generality and strong theoretical guarantees, both \tas and \stas have become fundamental algorithms. 
However, as noted in several works, the analysis of \tas is asymptotic in nature and does not offer insights into its non-asymptotic behavior \citep{barrier2022non,barrier2023contributions,jourdan2023non,poiani2024best,russo2025pure}. \citet{barrier2023contributions} suggest that this difficulty arises from instability in the sampling rule when data are scarce: the sampling rule employed by \tas can vary significantly early on, when estimated means have not yet concentrated around their expectations. The problem is even more pronounced for \stas. Intuitively, this is due to the fact that the algorithm is more complex, and that it samples the next arm in a \tas fashion. 

In this work, we address the following question:
\addtolength\leftmargini{-1.0cm}
\begin{quote}\centering\itshape
Can we characterize the non-asymptotic guarantees of \tas and \stas?
\end{quote}
We answer this question by providing the first non-asymptotic bounds for both Track-and-Stop and Sticky Track-and-Stop, shedding light on their behavior in the finite-confidence setting. Our results hold for arbitrarily structured problems and correct answer correspondences. 

Importantly, from an empirical side, the performance of \tas has been benchmarked several times \citep[\eg][]{degenne2019non,wang2021fast,jourdan2022top,barrier2022non}. The results have consistently shown that \tas obtains highly competitive sample complexity even in the moderate regime of $\delta$. Thus, the theoretical findings that we present here \emph{complement} these results, showing that \tas also enjoys finite-confidence guarantees. Furthermore, since \stas is, to the best of the authors' knowledge, the only algorithm in the literature that can solve arbitrary multiple answer problems, our work also provides the \emph{first finite-confidence analysis for this general class of problems}.

\subsection{Related Work}
\paragraph{Single-Answer Problems and \tas}
Since the work by \citet{garivier2016optimal}, which studied the unstructured bandit problem, several works have extended \tas to several structured problems \citep[\eg][]{moulos2019optimal,juneja2019sample,kocak2020best,kaufmann2021mixture,poiani2024best,kanarios2024cost}. 
In these works, the optimality analysis of \tas remained asymptotic, as they build upon the approach of \citet{garivier2016optimal}.
Beyond \tas, several other algorithms have been proposed in the literature that achieve asymptotic optimality for best-arm identification and/or single-answer problems \citep{degenne2019non,menard2019gradient,wang2021fast,barrier2022non,jourdan2022top,jourdan2023non}. Moreover, some of these works also establish non-asymptotic upper bounds on the expected number of samples required for their algorithm to stop \citep{degenne2019non,barrier2022non,jourdan2023non,wang2021fast}.
Among these studies, the works most closely related to ours are \citet{degenne2019non} and \citet{barrier2022non}, which propose two distinct approaches that solve the aforementioned instability issue of \tas. Precisely, \citet{degenne2019non} propose an optimistic version of \tas that incorporates confidence intervals within the sampling rule. This comes at the cost of solving a significantly more challenging optimization problem in deciding the next arm to query. In particular, while \tas bases its decision rule on a maximization problem of a convex function over the simplex, its optimistic version requires solving a max-max-min problem which, in general, does not admit efficient oracles.
\citet{barrier2022non}, instead, stabilizes the sampling rule by \quotes{skewing} its behavior toward uniform exploration when the amount of data collected is scarce. This modification, however, can lead to a decrease in the empirical performance w.r.t. the original version. Indeed, \citet{barrier2022non} show that the original version of \tas achieves better results in all instances in which the algorithms were compared. Hence, given these remarks, it remains important to understand whether finite confidence guarantees can be obtained for the original version of \tas.
In this sense, we note that both the analyses of \citet{degenne2019non} and \citet{barrier2022non} rely on specific properties of the algorithms they introduce, which makes them different from the original framework of \citet{garivier2016optimal}. To the best of our knowledge, our work is the first to provide non-asymptotic guarantees for the original \tas algorithm without requiring any substantial modifications to it.

\paragraph{Multiple-Answer Problems and \stas}
When there are multiple correct answers, \stas \citep{degenne2019pure} offers a solution to pure exploration problems while enjoying asymptotic optimality. 
However, to the best of our knowledge, there is no variant of \stas that achieves finite-confidence guarantees in arbitrary multiple-answer problems. More broadly, we are not aware of any algorithm providing finite-confidence guarantees for general multiple-answer problems. Prior work has largely focused on \emph{specific subclasses} of pure exploration problems. Among these, the most studied one is $\epsilon$-best arm identification \citep[e.g.,][]{even2002pac,kalyanakrishnan2012pac,karnin2013almost,kocak2021epsilon,jourdan2023varepsilon,jin2024optimal}. Nonetheless, none of these studies can be applied to the general pure exploration setting of \citet{degenne2019pure}.

\section{Background}\label{sec:background}

We focus on bandit problems $\bm\varphi = \{ \varphi_k\}_{k \in [K]}$ with $K \in \mathbb{N}$ arms, where $\varphi_k$ is a probability distribution with mean $\mu_k$. We denote by $\bm\mu = \{ \mu_k \}_{k \in [K]}$ the vector of the means of the distributions. 
As usual in the literature (see \eg \cite{garivier2016optimal,degenne2019non}), we focus on distributions that belong to a canonical exponential family\footnote{These include Gaussian with known variance and Bernoulli distributions. See \citet{cappe2013kullback}.} 
It is well known that such distributions are fully characterized by their means. For convenience, we will refer to a bandit model $\bm\varphi$ directly by its vector of means $\bm\mu$.
We denote by $\Theta \subseteq \mathbb{R}$ an open set that defines the possible means of the distributions. We consider the general case where $\bm\mu \in \Mcal \subseteq \Theta^K$. This allows to include in our analysis also structured settings such as Lipschitz \citep{wang2021fast} or unimodal bandits \citep{poiani2024best}. Indeed, since $\mathcal{M}$ is any subset of $\Theta^K$, it can directly encode the constraints imposed by these structures.\footnote{For completeness, we show in \Cref{app:structured} how to encode these structures through $\mathcal{M}$.} Moreover, we assume a finite answer space $\Ical$, along with access to a set-valued function $i^{\star}: \Mcal \mapstoto \Ical$, which maps each model $\bm\mu \in \Mcal$ to the set of all the answers that are correct for the bandit instance $\bm\mu$.

At each step $t\in\Naturals$, the learner chooses an action $A_t\in[K]$ and observes a sample $X_t\sim\varphi_{A_t}$. Let $\Fcal_t=\sigma(A_1, X_1, \dots A_t, X_t)$ be the $\sigma$-field generated by the interactions with the bandit model up to time $t$. Then, a pure exploration algorithm receives as input a confidence level $\delta \in (0, 1)$, and implements the following procedures: (i) a $\Fcal_{t-1}$-measurable \emph{sampling rule} which selects the action $A_t\in [K]$ based on the past observations, (ii) a \emph{stopping rule} $\tau_\delta$ which is a stopping time w.r.t.  $(\Fcal_t)_{t\in\Naturals}$ and controls the end of the data acquisition phase, and (iii) a  $\Fcal_{\tau_\delta}$-measurable \emph{recommendation rule} $\hat{\imath}_{\tau_\delta} \in \Ical$ that denotes the guess of the statistician for a correct answer for $\bm\mu$.
A pure exploration algorithm is \emph{$\delta$-correct} on $\mathcal{M}$ if it satisfies $\Prob_{\bm\mu}(\hat \imath_{\tau_{\delta}} \notin i^{\star}(\bm\mu))\le\delta$ for all $\bm\mu\in\Mcal$. The goal is building algorithms which are $\delta$-correct and that minimize the expected stopping time, \ie $\mathbb{E}_{\bm{\mu}} \left[ \tau_\delta \right] = \sum_{k \in [K]} \mathbb{E}_{\bm{\mu}}[N_{k}(\tau_\delta)]$,
where $N_{k}(t)$ is the number of samples collected for arm $k \in [K]$ up to time $t$. In the following, we denote by $\bm{N}(t)$ the vector $\left( N_1(t), \dots, N_K(t) \right)$.

\paragraph{Additional Notation} 
For a given set $\Xcal$, we denote by $\cl(\Xcal)$ its closure. Furthermore, for all $i \in \Ical$, we denote by $\lnot i = \{ \bm\lambda \in \Mcal: i \notin i^{\star}(\bm\lambda) \}$. In words, $\lnot i$ represents the set of bandit models for which $i$ is not a correct answer. Without loss of generality, we assume that for all $\bm\mu \in \Mcal$, there exists $i \in i^{\star}(\bm\mu)$ such that $\bm\mu \notin \cl(\neg i)$.\footnote{It is easy to verify from the lower bounds that, whenever this assumption is not satisfied, one obtains infinite sample complexity. This requirement is usually implicitly satisfied in the literature, \eg $\argmax_{k \in [K]} \mu_k$ is unique over $\Mcal$ in best-arm identification problems \citep{degenne2019pure}, and the different sets $\Mcal_j$ are \emph{open} and disjoint in the partition identification problem \citep{wang2021fast}.}
Furthermore, for distributions with means $p$ and $q$, we write $d(p,q)$ to denote their KL divergence. Moreover, for a distribution with mean $p$, we denote by $\nu_p$ the corresponding natural parameter within the exponential family. For $n \in \Naturals$, $\Delta_n$ denotes the $n$-dimensional simplex. 
Finally, consider two topological spaces $\Xcal$ and $\mathcal{Y}$, and consider a set-valued function $F: \mathcal{X} \mapstoto \mathcal{Y}$ that maps each $x \in \mathcal{X}$ to $F(x) \subseteq \mathcal{Y}$. We say that $F$ is upper hemicontinuous, if for all $x \in \Xcal$ and any open set $\mathcal{V} \subseteq \mathcal{Y}$ such that $F(x) \subseteq \mathcal{V}$, there exists a neighbourhood $\mathcal{U}$ of $x$ such that, for all $x' \in \mathcal{U}$, $F(x')$ is a subset of $\mathcal{Y}$ \citep{aubin1999set}.

\paragraph{Lower Bound for Single-Answer Problems}
Let us focus on the case where $i^{\star}(\bm\mu)$ is unique for all $\bm\mu \in \Mcal$. Lower bounds for these problems can be derived following the arguments presented in \cite{garivier2016optimal}. Specifically, for any $\delta$-correct algorithm, it holds that $\E_{\bm\mu}[\tau_{\delta}] \ge T^{\star}(\bm\mu) \log(1/(2.4\delta))$ (see \Cref{app:lb} for a formal statement), where:\begin{align}
    T^{\star}(\bm\mu)^{-1} & = \sup_{\bm\omega \in \Delta_K} \inf_{\bm\lambda \in \lnot i^{\star}(\bm\mu)} \sum_{k \in [K]} \omega_k d(\mu_k, \lambda_k) \label{eq:t-star} \\ 
    & = \sup_{\bm\omega \in \Delta_K} \max_{i \in \Ical} \inf_{\bm\lambda \in \neg i} \sum_{k \in [K]} \omega_k d(\mu_k, \lambda_k). \label{eq:t-star-v2}
\end{align}
$T^{\star}(\bm\mu)^{-1}$ can be interpreted as a max-min game where the max player plays a sampling strategy $\bm\omega$ to quickly identify the correct answer $i^{\star}(\bm\mu)$, and the min-player chooses a confounding instance $\bm\lambda \in \mathcal{M}$ where the correct answer changes \citep{degenne2019non}. The convex set of weights that attains the argmax in $T^{\star}(\bm\mu)^{-1}$ are denoted by $\bm\omega^{\star}(\bm\mu)$ and takes the name of \emph{oracle weights}. Here, convexity simply follows from the fact $T^{\star}(\bm\mu)^{-1}$ is a supremum over functions that are linear in $\bm\omega$. We note that we provided two expressions for $T^{\star}(\bm\mu)^{-1}$. \Cref{eq:t-star} is the one that most frequently appears in the literature (\eg \cite{garivier2016optimal}). \Cref{eq:t-star-v2} is a rewriting that allows to generalize the expression of $T^{\star}(\bm\mu)^{-1}$ to bandit models $\bm\mu$'s that fall outside $\Mcal$, as $i^{\star}(\bm\mu)$ is formally defined only for $\bm\mu \in \Mcal$. This is important from an algorithmic perspective as it allows us to generalize the definition of oracle weights to models that are not in $\Mcal$.\footnote{Indeed, we observe that an empirical estimate of $\bm\mu$ might not belong to $\Mcal$.} Furthermore, as we shall see, it will also play an important role in our analysis. Finally, since the problem is single-answer, then for all $\bm\mu \in \Mcal$ the $\argmax$ over the different answers is attained only at $i = i^{\star}(\bm\mu)$. 

\paragraph{Lower Bound for Multiple-Answer Problems}
Lower bounds for multiple-answer problems were established by \citet{degenne2019pure}. Specifically, the authors shows that, for any $\delta$-correct algorithm and any $\bm\mu \in \Mcal$, it holds that $\liminf_{\delta \to 0} \frac{\E_{\bm\mu}[\tau_\delta]}{\log(1/\delta)} \ge  T^{\star}(\bm\mu)$,
where $T^{\star}(\bm\mu)^{-1}$ is given by (formal statement in \Cref{app:lb}):
\begin{align}
    T^{\star}(\bm\mu)^{-1} & = \sup_{\bm\omega \in \Delta_K} \max_{i \in i^{\star}(\bm\mu)} \inf_{\bm\lambda \in \neg i} \sum_{k \in [K]} \omega_k d(\mu_k, \lambda_k) \label{eq:t-star-ma} \\
    & = \sup_{\bm\omega \in \Delta_K} \max_{i \in \Ical} \inf_{\bm\lambda \in \neg i} \sum_{k \in [K]} \omega_k d(\mu_k, \lambda_k) \label{eq:t-star-ma-v2}.
\end{align}
We have introduced two expressions for $T^{\star}(\bm\mu)^{-1}$: one that only applies to models within $\Mcal$ (\Cref{eq:t-star-ma}), and another that extends the definition to models outside $\Mcal$ (\Cref{eq:t-star-ma-v2}). While these results closely resemble those of single-answer problems, a few differences need to be highlighted. First, this lower bound only holds in the asymptotic regime of $\delta \to 0$. Second, the $\argmax$ in $T^{\star}(\bm\mu)^{-1}$ over the different answers can be attained at multiple points. Specifically, let $i_F(\bm\mu)$ be the set of answers that attain the $\argmax$, \ie $i_F(\bm\mu)=\argmax_{i \in \Ical} \sup_{\bm\omega \in \Delta_K} \inf_{\bm\lambda \in \neg i} \sum_{k \in [K]} \omega_k d(\mu_k, \lambda_k)$.
Then, while for single answer problems $|i_F(\bm\mu)|=|i^{\star}(\bm\mu)|=1$ for all $\bm\mu \in \Mcal$, in multiple answer problems it can happen that $|i_F(\bm\mu)| > 1$. %
Since it plays a crucial role in our results, we emphasize that the correspondence $\bm\mu \mapstoto i_F(\bm\mu)$ is upper hemicontinuous (Theorem 4 of \cite{degenne2019pure}).
Finally, we mention that the oracle weights $\bm\omega^{\star}(\bm\mu)$ are no longer a convex set when $|i_F(\bm\mu)|>1$. Instead, we have that $\bm\omega^{\star}(\bm\mu) = \bigcup_{i \in i_F(\bm\mu)} \bm\omega^{\star}(\bm\mu, \neg i)$, where each element $\bm\omega^{\star}(\bm\mu, \neg i) \coloneqq \argmax_{\bm\omega \in \Delta_K} \inf_{\bm\lambda \in \neg i} \sum_{k \in [K]} \omega_k d(\mu_k, \lambda_k)$ is a convex set.

\paragraph{Track-and-Stop}
Track-and-Stop \citep[\tas,][]{garivier2016optimal} works as follows. After a first phase where each arm $k \in [K]$ is pulled once,
\tas computes, at each round $t$, the empirical oracle weights $\bm\omega(t) \in \bm\omega^{\star}(\bmMuHat (t))$, where $\hat{\mu}_k(t) = N_k(t)^{-1} \sum_{s=1}^{t} \bm{1}\{ A_s = k \} X_s$ denotes the empirical estimate of $\mu_k$ at time $t$. Then, \tas applies a tracking procedure on $\{ \bm\omega(t) \}_{t}$ to select the next action. Specifically, the \emph{C-Tracking} procedure projects each $\bm\omega(t)$ onto $\Delta_K^{\epsilon_s} = \Delta_K \cap [\epsilon_s, 1]^K$ according to the $\ell_{\infty}$ norm. This projection takes the name of \emph{forced exploration}, as it ensures that $N_k(t) \gtrsim \sqrt{t}$ for all $k \in [K]$ for $\epsilon_t \approx t^{-1/2}$. The next action is selected as $A_{t+1} \in \argmax_{k \in [K]} \sum_{s=K}^{t} \tilde{\omega}_k(s) - N_k(t)$, where each $\tilde{\bm\omega}(s)$ denotes the projection of $\bm\omega(s)$. 
Regarding the stopping and recommendation rules, \tas halts as soon as $\max_{i \in \Ical} \inf_{\bm\lambda \in \lnot i} \sum_{k \in [K]} N_k(t) d(\hat{\mu}_k(t), \lambda_k) \ge \beta_{t,\delta}$, and recommends an index that attains the $\argmax$ in the stopping rule. 
By calibrating the threshold $\beta_{t,\delta}$ (typically, $\beta_{t,\delta} \approx \log(1/\delta) + K\log(t) $; see \eg \cite{kaufmann2021mixture} for a complete expression of $\beta_{t,\delta}$) one can prove that those stopping and recommendation rules yield $\delta$-correctness (both for single and multiple-answer problems) when paired with \emph{any} sampling rule.\footnote{For completeness, we report a formal statement and a proof in \Cref{lemma:correctness}.} \tas enjoys asymptotic optimality guarantees whenever $|i_F(\bm\mu)|=1$, \ie $\limsup_{\delta \to 0} \frac{\E_{\bm\mu}[\tau_\delta]}{\log(1/\delta)} \le T^{\star}(\bm\mu)$. However, this does not hold for $|i_F(\bm\mu)|>1$ \citep{degenne2019pure}. The reason is that its sampling rule ensures that the empirical pull strategy $\bm{N}(t)/t$ converges (on a good event) to the convex hull of the oracle weights, \ie $\inf_{\bm\omega \in \textup{conv}(\bm\omega^{\star}(\bm\mu))} \| \bm{N}(t)/t - \bm\omega \| \to 0$.\footnote{See Lemma 6 and Theorem 7 in \citet{degenne2019non}.} When $|i_F(\bm\mu)|=1$, this convex hull coincides with $\bm\omega^{\star}(\bm\mu)$, and this leads to optimality. However, this is not generally true in the context of multiple-answer problems.

\paragraph{Sticky Track-and-Stop} 
To solve this issue, \citet{degenne2019pure} proposed the Sticky Track-and-Stop (\stas) algorithm. The stopping and recommendation rules are the same used by \tas. As for the sampling rule, \stas defines a confidence region $C_t$ around $\bmMuHat(t)$, \ie $C_t = \{ \bm\lambda \in \Mcal: \sum_{k \in [K]} N_k(t) d(\hat{\mu}_k(t), \lambda_k) \le 8K \log(t) \}$,and computes a set $\Ical_t$ of candidate answers as follows: $\Ical_t = \bigcup_{\bm\lambda \in C_t} i_F(\bm\lambda)$. Then, \stas selects an answer $i_t \in \Ical_t$ according to some pre-specified total order over $\Ical$, and it computes $\bm\omega(t) \in \bm\omega^{\star}(\bmMuHat(t), \neg i_t)$ for the selected answer $i_t$. Finally, it selects the next action $A_t$ by applying the C-Tracking sampling rule over the sequence $\{ \bm\omega(t) \}_t$. The main idea behind \stas is that, due to the upper-hemicontinuity of $\bm\mu \mapstoto i_F(\bm\mu)$, the set $\Ical$ will eventually collapse (under a good event) to $i_F(\bm\mu)$ for sufficiently large $t$. Then, since $i_t$ is chosen according to a pre-specified total order over $\Ical$, $i_t$ will be fixed to some $\imath \in i_F(\bm\mu)$, and the C-Tracking sampling rule will ensure that $\inf_{\bm\omega \in \bm\omega^{\star}(\bm\mu, \neg \imath)} \| \bm{N}(t)/t - \bm\omega \| \to 0$. As shown by \citet{degenne2019pure}, this property leads asymptotic optimality both in single and multiple-answer problems.

\section{Non-Asymptotic Analysis For Track-and-Stop}\label{sec:finite-confidence}

First, we present two assumptions that we will use in our analysis.

\begin{assumption}[Sub-Gaussian Arms]\label{ass:subgauss}
    Arms belongs to a $\sigma^2$-sub-Gaussian exponential family, \ie
    for all $\mu, \mu' \in \Theta$, it holds $d(\mu, \mu') \ge \frac{(\mu - \mu')^2}{2\sigma^2}$.
\end{assumption}

\begin{assumption}[Bounded parameters]\label{ass:bounded}
    There exists $[\muMin, \muMax] \subset \Theta$ such that $\Mcal \subset [\muMin, \muMax]$.
\end{assumption}

Both \Cref{ass:subgauss,ass:bounded} are mild requirements that have been frequently adopted in the literature; see \eg \cite{degenne2019non,degenne2020structure,jourdan2021efficient,poiani2024best}. \Cref{ass:subgauss} is used primarily for concentration arguments. \Cref{ass:bounded} implies that, for any two distributions $p$ and $q$ within $[\muMin, \muMax]$ it holds that $d(p,q) \le L$ and $|\nu_{p} - \nu_{q}| \le D$, for some constants $L$ and $D$. These two properties are the main reason for introducing \Cref{ass:bounded} in our analysis.\footnote{We refer the interested reader to \Cref{app:assumptions} for further details.}

Before introducing our result, we make a minor modification to the \tas algorithm that allows for a simpler analysis. Specifically, instead of computing $\bm\omega(t) \in \bm\omega^{\star}(\bmMuHat(t))$, it computes $\bm\omega(t) \in \bm\omega^{\star}(\bmMuTilde(t))$, where $\bmMuTilde(t)$ denotes the orthogonal projection of $\bmMuHat(t)$ onto $[\muMin, \muMax]^K$.\footnote{Note that, since $\Mcal$ is know by definition, so are $\muMin$ and $\muMax$. In other words, this kind of modification does not require additional knowledge of the problem.} This modification is only required to ensure that $d(\tilde{\mu}_k(t), \cdot)$ ``well-behaves'' whenever $t$ is small. Such projection has already been adopted in sampling rules for algorithms that provide finite-confidence guarantees, see, \eg the regret minimization approach presented in   \cite{degenne2019non}. More formally, its purpose is ensuring that $d(\tilde{\mu}_k(t), \lambda)$ is bounded for all steps and any $\lambda \in [\muMin, \muMax]$. We note that this modification is only needed to handle pathological cases that might arises when dealing with arbitrary canonical exponential families and it is not needed, \eg when the family of distributions is Gaussian. Later in this section, we discuss how to drop the projection step and how this affects the resulting guarantees. We are now ready to state our finite-confidence result for \tas. 
\begin{theorem}[Non-Asymptotic Bound for \tas]\label{theo:proj-tas-finite-confidence}
    Let $i^{\star}(\cdot)$ be single-valued, and suppose that \Cref{ass:subgauss} and \Cref{ass:bounded} hold. Then, the expected stopping time of \tas satisfies $\E_{\bm\mu}[\tau_\delta] \le 2eK + 10 K^4 + T_0(\delta)$,
    where $T_0(\delta)$ is given by
    \begin{align}\label{eq:t0-delta-tas}
        T_0(\delta) = \inf\left\{t \in \Naturals: \beta_{t,\delta} \le t T^{\star}(\bm\mu)^{-1} - g(t)    \right\},
    \end{align}
    where $g(t) = 64\sigma DLK^2\log(K)\sqrt{t \log^2(t)} + 16\sigma D\sqrt{K t^{3/2} \log(t)}$.
\end{theorem}

\Cref{theo:proj-tas-finite-confidence} provides the finite-confidence bound on the performance of \tas. First, we note that the upper bound is expressed as a sum of three terms, \ie $2eK$, $10K^4$ and $T_0(\delta)$. The first two $\delta$-independent terms are artifact of the analysis and their origin is detailed in the proof sketch provided below.
The last and more important term, $T_0(\delta)$, is a function of $\delta$, which essentially captures how quickly the quantity $\max_{i \in \Ical} \inf_{\bm\lambda \in \neg i} \sum_{k \in [K]} N_k(t) d(\hat{\mu}_k(t), \lambda_k)$ is approaching the stopping threshold $\beta_{t,\delta}$.
Indeed, $t T^{\star}(\bm\mu)^{-1} - g(t)$ is essentially a lower bound (under a good event) on the aforementioned optimization problem: when this quantity exceeds $\beta_{t,\delta}$, \tas stops.  In other words, $T_0(\delta)$ measures how fast \tas is gathering information to discriminate $i^{\star}(\bm\mu)$ from all the other candidate answers. 
By re-arranging the condition in \Cref{eq:t0-delta-tas}, \ie $\beta_{t,\delta} + g(t) \le t T^{\star}(\bm\mu)^{-1}$, we can see that the r.h.s.~grows linearly in $t$, while the l.h.s.~is growing sub-linearly with a rate of $\mathcal{O}(\log(1/\delta) + t^{3/4})$.\footnote{Here, we plugged in $\beta_{t,\delta} \approx \log(1/\delta) + K\log(t)$.} This ensures that $T_0(\delta)$ is finite and that \Cref{theo:proj-tas-finite-confidence} recovers the asymptotic optimality of \tas whenever $\delta \to 0$. Finally, although $T_0(\delta)$ is defined somehow implicitly, in \Cref{app:explicit-bound} we derive a further upper bound that highlights that the scaling is $T^{\star}(\bm\mu) \log(1/\delta)$ up to polylogarithmic factors and constant terms.

\paragraph{Proof Sketch}
As we discussed, the original analysis of \tas is asymptotic in nature. In contrast, we follow a different path which is inspired by finite-confidence analysis in the literature, \eg \cite{degenne2019non,jourdan2023non}. In particular, we conduct the analysis under a sequence of good events $\{\Ecal_t \}_{t}$. Specifically, we consider $\Ecal_t = \left\{ \forall s \in \left[ \lceil \sqrt{t} \rceil , t \right]: \sum_{k \in [K]} N_k(s) d(\hat{\mu}_k(s), \mu_k) \le 8K \log(s) \right\}$. This sequence of events has two desirable properties. First, one can show that $ \sum_{t=3}^{+\infty}\mathbb{P}_{\bm\mu}( \Ecal_t^c) \le 2eK$ (see \Cref{lemma:good-event}). Second, as we discuss below, there exists a time $\bar{T}$ such that for all $t \ge \bar{T}$, $\Ecal_t$ implies stopping, namely $\Ecal_t \subseteq \{ \tau_\delta \le t \}$. Using these two properties one obtains that $\E_{\bm\mu}[\tau_\delta] \le \bar{T} + 2eK$ (see \Cref{lemma:expectation}). In the remainder of the proof, we will show how $\bar{T} \coloneqq T_0(\delta) +  10K^4 $ satisfies the requirement mentioned earlier.  Before doing that, we introduce some additional notation. Recall that $\bm\omega(s) \in \argmax_{\bm\omega \in \Delta_K} \max_{i \in \Ical} \inf_{\bm\lambda \in \neg i} \sum_{k \in [K]} \omega_k d(\tilde{\mu}_k(s), \lambda_k)$. Then, we denote by $i_s \in \Ical$ any answer that attains the argmax when paired with $\bm\omega(s)$. 

Now, the key idea is analyzing the stopping rule of \tas and, in particular, lower bounding $\max_{i \in \Ical} \inf_{\bm\lambda \in \neg i} \sum_{k \in [K]} N_k(t) d(\hat{\mu}_k(t), \lambda_k)$ to obtain $t T^{\star}(\bm\mu)^{-1} - g(t)$. To this end, as we shall see, the crucial step is approximating (up to a sublinear in $t$ factor) the max-min problem of the stopping rule with what \tas uses in the sampling rule in each round $s \ge \sqrt{t}$, \ie $\inf_{\bm\lambda \in \neg i_s} \sum_{k \in [K]} \omega_k(s) d(\tilde{\mu}_k(s), \lambda_k)$. Further comments on this are provided right after the proof sketch. Now, for any $t \ge  10K^4  $,\footnote{This requirement is due to some technical step that is used at the end of the proof.} if \tas has not stopped at time $t$, then the following holds:
\begin{align*}
    \beta_{t,\delta} & > \inf_{\bm\lambda \in \neg i^{\star}(\bm\mu)} \sum_{k \in [K]} N_k(t) d(\hat{\mu}_k, \lambda_k) \tag{Stopping Rule} \\
    & \gtrsim \sum_{s=1}^t \inf_{\bm\lambda \in \neg i^{\star}(\bm\mu)} \sum_{k \in [K]}  \omega_k(s) d(\mu_k, \lambda_k) - \widetilde{\mathcal{O}}(\sqrt{t}) \tag{$\Ecal_t$ + C-Tracking} \\
    & \ge \sum_{s=1}^t \inf_{\bm\lambda \in \neg i_s} \sum_{k \in [K]}  \omega_k(s) d(\mu_k, \lambda_k) - \widetilde{\mathcal{O}}(\sqrt{t}).
\end{align*}
Here, C-tracking ensures that $\bm{N}(t) \approx \sum_{s=1}^t \bm\omega(s)$, and under the event $\Ecal_t$ we have $d(\hat{\mu}_k(t), \cdot) \approx d(\mu_k, \cdot)$.
In the last step, we have used that if $i_s = i^{\star}(\bm\mu)$ then the claim is trivial, and if $i_s \ne i^{\star}(\bm\mu)$, then, $\bm\mu \in \neg i_s$. We observe that this argument explicitly relies on the fact that $i^{\star}(\bm\mu)$ is single-valued.\footnote{This is important to be noted, otherwise it might seems that \tas achieves asymptotically optimal results even in problems with multiple correct answers.} 
Now, we analyze the information accumulated by \tas by lower bounding $\sum_{s=1}^t \inf_{\bm\lambda \in \neg i_s} \sum_{k \in [K]}  \omega_k(s) d(\mu_k, \lambda_k)$. Using the definition of $\Ecal_t$ and the fact that $\tilde{\mu}_k(s) \in [\muMin, \muMax]$,\footnote{This is needed to upper bound $d(\tilde{\mu}_k(s), \lambda_k) - d(\mu_k, \lambda_k)$. Indeed, $d(\tilde{\mu}_k(s), \lambda_k) - d(\mu_k, \lambda_k) \le  (\nu_{\tilde{\mu}_k(s)}- \nu_{\lambda_k}) |\tilde{\mu}_k(s) - \mu_k| \le D |\tilde{\mu}_k(s) - \mu_k|$ since $\bmMuTilde(s), \bm\lambda \in [\muMin, \muMax]$.} we have that 
\begin{align*}
    \sum_{s=1}^t \inf_{\bm\lambda \in \neg i_s} \sum_{k \in [K]}  \omega_k(s) d(\mu_k, \lambda_k) & \gtrsim \sum_{s\ge\sqrt{t}} \inf_{\bm\lambda \in \neg i_s} \sum_{k \in [K]}  \omega_k(s) d(\tilde{\mu}_k(s), \lambda_k) - \widetilde{\mathcal{O}}(\sqrt{t}). 
\end{align*}
We have reached our goal of lower bounding the stopping rule of \tas with its sampling rule. Using the definition of $i_s$ and $\bm\omega(s)$, this allows for the following inequalities:
\begin{align*}
    \sum_{s\ge\sqrt{t}} \inf_{\bm\lambda \in \neg i_s} \sum_{k \in [K]}  \omega_k(s) d(\tilde{\mu}_k(s), \lambda_k) & = \sum_{s\ge\sqrt{t}} \sup_{\bm\omega \in \Delta_K} \max_{i \in \Ical} \inf_{\bm\lambda \in \neg i} \sum_{k \in [K]}  \omega_k(s) d(\tilde{\mu}_k(s), \lambda_k) \\
    & \ge \sum_{s \ge \sqrt{t}} \inf_{\bm\lambda \in \neg i^{\star}(\bm\mu)} \sum_{k \in [K]}  \omega_k^{\star} d(\tilde{\mu}_k(s),\lambda_k) \tag{for $\bm\omega^{\star} \in \bm\omega^{\star}(\bm\mu)$} \\
    & = (t-\sqrt{t} - 1)T^{\star}(\bm\mu)^{-1} - \widetilde{\mathcal{O}}(\sqrt{t^{3/2}}) \tag{By $\Ecal_t$},
\end{align*}
where the second step holds for any $\bm\omega^{\star} \in \bm\omega^{\star}(\bm\mu)$ and the last one requires an algebraic step that requires $t \ge  10K^4 $. Intuitively, however, this last step is still using the fact that $d(\tilde{\mu}_k(s), \cdot) \approx d(\mu_k, \cdot)$ under the good event. 
Chaining together all the terms within the $\widetilde{O}(\cdot)$ yields the desired result.

\paragraph{The proof idea}
As anticipated above, the main idea is approximating up to a sub-linear in $t$ factor, the condition used in the stopping rule with the quantity $\sum_{s\ge\sqrt{t}} \inf_{\bm\lambda \in \neg i_s} \sum_{k \in [K]}  \omega_k(s) d(\tilde{\mu}_k(s), \lambda_k)$, which is what \tas uses in its sampling rule. Once this is done, we can use the definition of $\bm\omega(s)$ and $i_s$ to introduce the optimal weights $\bm\omega^{\star}$ for the underlying unknown problem and the infimum over $\neg i^{\star}(\bm\mu)$. Importantly, we observe that the generalization of $\bm\omega^{\star}$ that we provided in \Cref{eq:t-star-v2} played a crucial role. Finally, by upper-bounding the difference between $d(\tilde{\mu}_k(s), \lambda_k)$ and $d({\mu}_k, \lambda_k)$, we introduce $T^{\star}(\bm\mu)^{-1}$, which is the desired quantity as it allows to recover the asymptotic optimality.

\paragraph{Removing the projection step}
We discuss how to obtain finite-confidence guarantees for a version of \tas that does not use projection in the sampling rule, \ie exactly the version of \tas by \citet{garivier2016optimal}.
Before that, we make a remark on \Cref{ass:bounded}. Let $\bm\mu \in \Mcal$ and let $F_k= \min\{ |\mu_k - \muMin|, |\mu_k-\muMax| \}$ and $F = \min_{k \in [K]} F_k$.
Then, since $\Theta$ is an open interval and since $[\muMin, \muMax]$ is closed, it follows that $F_k > 0~\forall k\in[K]$, and thus $F > 0$. 
That being said, the simplest way to analyze \tas without projection follows by noticing that there exists a time $T_{\Mcal} \in \Naturals$ such that, for all $t \ge T_{\Mcal}$, on $\Ecal_t$, it holds that $\bmMuHat(s) \in [\muMin, \muMax]$ for all $s \ge \sqrt{t}$ (see \Cref{lemma:mu-hat-good-region} in \Cref{app:finite-confidence-proj-tas}). $T_{\Mcal}$ depends only $\Mcal$ and its distance $F$ from the interval $[\muMin, \muMax]$; precisely:
\begin{align}\label{eq:t-mcal}
    T_{\Mcal} = \max \left\{  10K^4, \inf \left\{ n \in \Naturals: \sqrt{\frac{64\sigma^2K \log(n)}{\sqrt{\sqrt{n}+K^2}-2K}} \le F \right\} \right\}.
\end{align}
This allows us analyze the stopping time under a good event in the same way that we did above. Indeed, it is sufficient that $d(\hat{\mu}_k(s), \cdot)$ well-behaves only at steps $s \ge \sqrt{t}$.
Thus, the only difference with respect to \Cref{theo:proj-tas-finite-confidence} would be the additional term $T_{\Mcal}$ in the upper bound of $\E_{\bm\mu}[\tau_\delta]$.

\section{Non-Asymptotic Analysis For Sticky Track-and-Stop}\label{sec:stas}

In this section, we present the finite-confidence analysis of \stas. As for \tas, we will rely on \Cref{ass:subgauss,ass:bounded}. %
Furthermore, for reasons similar to those discussed above, we consider a slightly modified version of \stas that incorporates a projection into its sampling rule.
Specifically, the algorithm computes $\bm\omega(s) \in \bm\omega^{\star}(\tilde{\bm\mu}(s), \neg i_s)$.\footnote{Using the same argument that we discussed in \Cref{sec:finite-confidence}, it is possible to analyze the version of \stas that does not use projection. The sample complexity results differ only by the additional term $T_{\mathcal{M}}$.} The following theorem summarizes our result.

\begin{theorem}[Non-Asymptotic Bound for Sticky-\tas]\label{theo:stas-finite-confidence}
Suppose that \Cref{ass:subgauss} and \Cref{ass:bounded} hold. Let $\epsilon_{\bm\mu} > 0$ be any number such that, for all $\bm\mu': \|\bm\mu - \bm\mu'\|_{\infty} \le \epsilon_{\bm\mu}$, it holds that $i_F(\bm\mu') \subseteq i_F(\bm\mu) \cup (\Ical \setminus i^{\star}(\bm\mu))$, and let $T_{\bm\mu} \in \Naturals$ be defined as follows:
\begin{align*}
    T_{\bm\mu} = \max \left\{  10K^4, \inf \left\{ n \in \Naturals:  \sqrt{\frac{64K\sigma^2 \log(n)}{\sqrt{\sqrt{n}+K^2}-2K}} \right\} \le \epsilon_{\bm\mu} \right\}.
\end{align*}
Then, it holds that $\E_{\bm\mu}[\tau_\delta] \le  2eK + 10K^4  + T_0(\delta) $, where $T_0(\delta)$ is given by
\begin{align*}
    T_0(\delta) = \inf \left\{ t \in \Naturals: \beta_{t,\delta} \le \left( t  -T_{\bm\mu} \right) T^{\star}(\bm\mu)^{-1} - g(t)  \right\},
\end{align*}
where $g(t) = 80\sigma DLK^2\log(K)\sqrt{t \log^2(t)} + 32\sigma D\sqrt{K t^{3/2} \log(t)}$.
\end{theorem}

\Cref{theo:stas-finite-confidence} provides a finite-confidence bound for \stas in multiple-answer problems. As one can notice, the result is similar in nature to what we presented for \tas in \Cref{theo:proj-tas-finite-confidence}. In particular, the expression of $T_0(\delta)$ is similar to that of \tas, and, for the same reasons outlined in \Cref{sec:finite-confidence}, this allows us to  recover the asymptotic optimality guarantees of \cite{degenne2019pure} whenever $\delta \to 0$.\footnote{As we did for \tas, in \Cref{app:explicit-bound} we provide an explicit upper bound on $T_0(\delta)$.}
The main difference between \Cref{theo:proj-tas-finite-confidence} and \Cref{theo:stas-finite-confidence} is the presence of an additional problem-dependent constant $T_{\bm\mu} T^{\star}(\bm\mu)^{-1}$ within the expression of $T_0(\delta)$. As our proof will reveal, $T_{\bm\mu}$ is the time that is needed by \stas (under the good event) to distinguish $i_F(\bm\mu) \cup \left( \Ical \setminus i^{\star}(\bm\mu) \right)$ from $i^{\star}(\bm\mu) \setminus i_F(\bm\mu)$. In other words, from that point on, under $\Ecal_t$, all the candidate models $\bm\mu'$ within the confidence region $C_s$ satisfy $i_F(\bm\mu') \subseteq i_F(\bm\mu) \cup \left( \Ical \setminus i^{\star}(\bm\mu) \right)$ for all $s \ge \sqrt{t}$. Indeed, whenever $t \ge T_{\bm\mu}$ it will be possible to link the stopping rule to $T^{\star}(\bm\mu)^{-1}$ (with some sub-linear terms) as we did for \tas.\footnote{It is interesting to note that in our proof we are not using $\Ical_s = i_F(\bm\mu)$, \ie that \stas has actually \quotes{sticked} to an answer in $i_F(\bm\mu)$. Instead, it is sufficient that $\Ical_s$ excludes answers in $i^{\star}(\bm\mu)$ that are not within $i_F(\bm\mu)$.} 

Finally, we comment on the nature of $T_{\bm\mu}$. In particular, the existence of $\epsilon_{\bm\mu} > 0$ is guaranteed by the upper hemicontinuity of the set-valued function $i^{\star}(\bm\mu)$. Our claim holds for any $\epsilon_{\bm\mu}$ that satisfies $\forall\bm\mu': \|\bm\mu - \bm\mu'\|_{\infty} \le \epsilon_{\bm\mu} \implies i_F(\bm\mu') \subseteq i_F(\bm\mu) \cup (\Ical \setminus i^{\star}(\bm\mu))$, and, consequently, the tightest bound is obtained for the largest possible $\epsilon_{\bm\mu}$. To conclude, we observe that the definition of $\epsilon_{\bm\mu}$ is intrinsic to the definition of the task at hand, \ie $i^\star(\bm\mu)$. Once the task is fixed, it might be possible to obtain a more explicit characterization of this quantity. Consider, for instance, the relevant case of an $\epsilon$-best arm identification problem. Here, it holds that $i_F(\bm\mu) = \argmax_{i \in [K]} \mu_i$ \citep[see, \eg][]{garivier2021nonasymptoticsequentialtestsoverlapping,jourdan2023varepsilon}. Let $\Delta_{\bm\mu} = \min_{i \notin i_F(\bm\mu)} \mu_{\star} - \mu_i$ where $\mu_\star$ is the value of any optimal arm. Then, whenever $\epsilon_{\bm\mu} < \Delta_{\bm\mu}$, we have that $i_F(\bm\mu') \subseteq i_F(\bm\mu)$. Hence, \Cref{theo:stas-finite-confidence} holds, for instance, with $\epsilon_{\bm\mu}= \Delta_{\bm\mu}/2$.

Now, we present a proof sketch of the result.

\paragraph{Proof Sketch}
As we did for \tas, the analysis is carried out under a sequence of good events  $\{ \Ecal_t \}_t$ which are exactly the ones that we considered when proving \Cref{theo:proj-tas-finite-confidence}. As above, we will show that for $\bar{T}= 10K^4  + T_0(\delta)$ and $t \ge \bar{T}$, we have that $\Ecal_t \subseteq \{ \tau_\delta \le t \}$. As a consequence, $\E_{\bm\mu}[\tau_\delta] \le 2eK + 10K^4 + T_0(\delta)$.
To do this, the main idea is lower bounding the condition used in the stopping rule with what \stas uses in the sampling rule. First, we state an intermediate result, which is a consequence of (i) the forced exploration of \stas, (ii) the definition of the region $C_t$ of candidate models, and (iii) the upper hemicontinuity of the set-valued function $i_F(\bm\mu)$. Specifically, in \Cref{lemma:good-answers} we prove that, $\forall t \ge T_{\bm\mu}$, on $\Ecal_t$, it holds that:
\begin{align}\label{eq:good-answer}
    i_F(\bm\mu') \subseteq i_F(\bm\mu) \cup \left( \Ical \setminus i^{\star}(\bm\mu) \right) \quad \forall s \ge \sqrt{t} \text{ and } \bm\mu' \in C_s.
\end{align}
Indeed, by upper hemicontinuity, models $\bm\mu'$ similar to $\bm\mu$ have answers in $i_F(\bm\mu')$ which are \quotes{close} to the ones in $i_F(\bm\mu)$, and models in $C_t$ shrink toward $\bm\mu$ due to the forced exploration of the algorithm.

We now analyze the amount of information that is gathered by \stas under the good event $\Ecal_t$. Let $t \ge  10K^4  $ and let $\widetilde{T} = \max \{ \lceil \sqrt{t} \rceil, T_{\bm\mu} \}$. Denote by $\imath \in \Ical$, the answer  that is selected from $i_F(\bm\mu)$ by the pre-specified total order over $\Ical$. Then, for $t \ge  10K^4  $, if \stas has not stopped at $t$, we have that:
\begin{align}
    \beta_{t,\delta} & \gtrsim \sum_{s=1}^t \inf_{\bm\lambda \in \neg \imath} \sum_{k \in [K]}  \omega_k(s) d(\mu_k, \lambda_k) - \widetilde{\mathcal{O}}(\sqrt{t}) \label{eq:stats-proof-1} \\
    & \ge \sum_{s = \widetilde{T}}^t \inf_{\bm\lambda \in \neg i_s} \sum_{k \in [K]} \omega_k(s) d(\mu_k, \lambda_k) - \widetilde{\mathcal{O}}( \sqrt{t}), \label{eq:stats-proof-2}
\end{align}
where the first step is due to concentration and C-Tracking, and the second one uses \Cref{eq:good-answer}. Indeed, for $s \ge \widetilde{T}$, either $i_s = \imath$ (and in this case the claim is trivial), or $i_s \ne \imath$. In this second case, from \Cref{eq:good-answer} we have that $i_s \notin i^{\star}(\bm\mu)$ and, therefore, $\bm\mu \in \neg i_s$. Now, by concentration arguments (\ie $d(\tilde{\mu}_k(s), \cdot) \approx d(\mu_k, \cdot)$), and using the definition of $\bm\omega(s)$, we have that:
\begin{align*}
    \sum_{s = \widetilde{T}}^t \inf_{\bm\lambda \in \neg i_s} \sum_{k \in [K]} \omega_k(s) d(\mu_k, \lambda_k) & \gtrsim \sum_{s = \widetilde{T}}^t \inf_{\bm\lambda \in \neg i_s} \sum_{k \in [K]} \omega_k(s) d(\tilde{\mu}_k(s), \lambda_k) - \widetilde{\mathcal{O}}( \sqrt{t}) \\
    & = \sum_{s = \widetilde{T}}^t \max_{\bm\omega \in \Delta_K} \inf_{\bm\lambda \in \neg i_s} \sum_{k \in [K]} \omega_k d(\tilde{\mu}_k(s), \lambda_k) - \widetilde{\mathcal{O}}(\sqrt{t}).
\end{align*}
The next step is crucial for relating the amount of gathered information to $T^{\star}(\bm\mu)^{-1}$. Let $\bm\mu'(s) \in C_s$ be such that $i_s \in i_F(\bm\mu'(s))$. From concentration arguments and the definition of $\bm\mu'(s)$, we have that: 
\begin{align*}
    \sum_{s = \widetilde{T}}^t \max_{\bm\omega \in \Delta_K} \inf_{\bm\lambda \in \neg i_s} \sum_{k \in [K]} \omega_k d(\tilde{\mu}_k(s), \lambda_k)  & \gtrsim \sum_{s = \widetilde{T}}^t \max_{\bm\omega \in \Delta_K} \inf_{\bm\lambda \in \neg i_s} \sum_{k \in [K]} \omega_k d({\mu}'_k(s), \lambda_k) - \widetilde{\mathcal{O}}\left(  \sqrt{t^{3/2}} \right) \\
    &\hspace{-.6cm} = \sum_{s = \widetilde{T}}^t \max_{\bm\omega \in \Delta_K} \max_{i \in \Ical} \inf_{\bm\lambda \in \neg i} \sum_{k \in [K]} \omega_k d({\mu}'_k(s), \lambda_k) - \widetilde{\mathcal{O}}\left( \sqrt{t^{3/2}} \right) \\
    & \hspace{-.6cm}\ge \sum_{s = \widetilde{T}}^t \inf_{\bm\lambda \in \neg \imath} \sum_{k \in [K]} \omega_k^{\star} d({\mu}'_k(s), \lambda_k) - \widetilde{\mathcal{O}}\left(  \sqrt{t^{3/2}} \right),
\end{align*}
for any $\bm\omega^{\star} \in \bm\omega^{\star}(\bm\mu, \neg \imath)$. The first step follows by observing that $d({\tilde{\mu}_k(s), \cdot})$ can be upper-bounded by $d({\hat{\mu}_k(s), \cdot})$, and $d({\hat{\mu}_k(s), \cdot}) \approx d(\mu'(s), \cdot)$ since $\bm\mu'(s) \in C_s$ by definition.
Then, the proof is simply concluded by noticing that:
\begin{align*}
    \sum_{s = \widetilde{T}}^t \inf_{\bm\lambda \in \neg \imath} \sum_{k \in [K]} \omega_k^{\star} d({\mu}'_k(s), \lambda_k) & \gtrsim \sum_{s = \widetilde{T}}^t \inf_{\bm\lambda \in \neg \imath} \sum_{k \in [K]} \omega_k^{\star} d({\mu}_k, \lambda_k)  - \widetilde{\mathcal{O}}( \sqrt{t^{3/2}}) \\
    & = (T-\widetilde{T}) T^{\star}(\bm\mu)^{-1} - \widetilde{\mathcal{O}}( \sqrt{t^{3/2}}),
\end{align*}
where the first step follows from concentration arguments and the fact that $\bm\mu'(s) \in C_s$. Rearrenging all the terms yields the desired result.

\paragraph{The proof idea}
As for \tas, the general idea is approximating with sub-linear terms the stopping rule with $ \sum_{s = \widetilde{T}}^t \inf_{\bm\lambda \in \neg i_s} \sum_{k \in [K]} \omega_k(s) d(\tilde{\mu}_k(s), \lambda_k)$, that is what \stas uses in its sampling rule. Now, there are two key differences with respect to \Cref{theo:proj-tas-finite-confidence}. First, to reach such objective we need to consider sufficiently large timesteps, \ie $s \ge \widetilde{T}$. The issue is that when $s$ is small, the \stas sampling rule has no control over the selected answers $i_s$ (apart from a generic total order over $\Ical$). This does not allow to easily switch from $\neg \imath$ to $\neg i_s$, \ie the step from \Cref{eq:stats-proof-1} to \Cref{{eq:stats-proof-2}}. Second, once we have obtained $\sum_{s = \widetilde{T}}^t \inf_{\bm\lambda \in \neg i_s} \sum_{k \in [K]} \omega_k(s) d(\tilde{\mu}_k(s), \lambda_k)$, this does not allow us to directly introduce $\bm\omega^{\star}(\bm\mu, \neg \imath)$ and $\neg \imath$ as we did for \tas. An intermediate step is necessary. This requires studying the difference between $d(\tilde{\mu}_k(s), \cdot)$ and $d(\mu'_k(s), \cdot)$. The reason is that $i_s$ is an answer that attains the $\argmax$ only when paired with a model $\bm\mu'(s) \in C_s$ such that $i_s \in i_F(\bm\mu'(s))$.

\paragraph{\stas in single-answer problems}
Whenever $i^{\star}(\bm\mu)$ is single-valued, the dependency on $T_{\bm\mu}$ can be removed as the step from \Cref{eq:stats-proof-1} and \Cref{eq:stats-proof-2} follows directly from the fact that $|i^{\star}(\bm\mu)|=1$ (as we did for \tas). Thus, one would obtain a result identical to \Cref{theo:proj-tas-finite-confidence} (\ie the same bound up to constant multiplicative terms). Nonetheless, we actually note that the two proofs are still different, and the reason is the different sampling rules adopted by the two algorithms. Specifically, in \stas $i_s$ is an answer in $i_F(\bm\mu'(s))$ for some $\bm\mu'(s) \in C_s$ and $\bm\omega(s) \in \bm\omega^{\star}(\tilde{\bm\mu}(s), \neg i_s)$. On the other hand, \tas directly uses $\bm\omega(s) \in \bm\omega^{\star}(\tilde{\bm\mu}(s), \neg i_s)$ for $i_s \in i_F(\tilde{\bm\mu}(s))$. 

\paragraph{On the behavior of \stas}
It is interesting to observe that our proof differs significantly from the one of asymptotic optimality by \citet{degenne2019pure}. Beyond the obvious distinction (\ie our analysis is non-asymptotic), we also note that the proof of \Cref{theo:stas-finite-confidence} does not rely on what $\bm{N}(t)/t$ converges to, nor does it exploit the convexity of the set $\bm\omega(\bm\mu, \neg \imath)$, which instead were key components in the analysis of \cite{degenne2019pure}. Instead, we only reason in terms of \quotes{information} collected by \stas by analyzing \emph{values} of functions of the form $\sum_{k \in [K]} \omega_k d(\cdot, \cdot)$. 

\section{Conclusion}\label{sec:conclusions}
This work provided the first finite-confidence characterization of the performance of Track-and-Stop and Sticky Track-and-Stop, two general algorithms that are able to solve optimally a large spectrum of pure exploration settings. Overall, we solve two open problems in the literature. First, \Cref{theo:proj-tas-finite-confidence} sheds light on the finite-confidence guarantees of \tas, thus providing theoretical support on why the algorithm usually enjoys good performance in practice. Secondly, \Cref{theo:stas-finite-confidence} provides the first finite-confidence guarantees for the general multiple-answer setting.   
To conclude, we note that our results (\Cref{theo:proj-tas-finite-confidence} and \Cref{theo:stas-finite-confidence}) have simple and natural proofs, and they both recover the asymptotic optimality whenever $\delta$ goes to $0$.

Several questions remain open. For instance, is it possible to improve the finite-confidence analysis of Sticky Track-and-Stop by removing the presence of the problem constant $T_{\bm\mu}$? We conjecture that this would require to slightly modify the sampling rule. Indeed, by selecting answers $i_t$ more strategically than using any total order over $\Ical$ might lead to stronger finite-confidence results and, eventually, more competitive performance. Furthermore, there remains a gap between lower and upper bounds in the finite-confidence regime \citep{degenne2019non,wang2021fast,barrier2022non,jourdan2023non,jourdan2023varepsilon}. Future work should focus on analyzing in greater details the finite-confidence regime of general pure exploration problems. Here, one could draw inspiration from the works of \cite{simchowitz2017simulator,al2022complexity,poiani2024best}, where the authors shed some light on the challenges in developing lower bounds and algorithms that enjoy tight dependencies not only in $\log(1/\delta)$, but also in other parameters of the instance (\eg $K$). Finally, \citet{poiani2025pure} recently extended \tas and \stas to problems where the set of correct answers is possibly infinite. However, their algorithm only attains asymptotic optimality, and our analysis cannot be straightforwardly extended to this setting. Future works, should focus on closing this gap. Here, it is relevant to mention the recent work of \citet{osogami2025optimal}, where the authors proposed an algorithm that enjoys finite confidence guarantees in a specific infinite answer problem, \ie that of estimating, up to an accuracy $\epsilon$, the value of the optimal arm.

\newpage

\subsubsection*{Acknowledgments}

Funded by the European Union. Views and opinions expressed are however those of the author(s) only and do not necessarily reflect those of the European Union or the European Research Council Executive Agency. Neither the European Union nor the granting authority can be held responsible for them.

Work supported by the Cariplo CRYPTONOMEX grant and by an ERC grant (Project 101165466 — PLA-STEER).

\subsubsection*{Reproducibility Statement}
The nature of this work is theoretical. We precisely stated and discussed the assumptions that are required to derive our result in the main text (see \Cref{ass:subgauss,ass:bounded}), and we further discussed the assumptions in \Cref{app:assumptions}. In the main text, we provided proof sketches for both \Cref{theo:proj-tas-finite-confidence} and \Cref{theo:stas-finite-confidence}, and we included complete proofs in \Cref{app:finite-confidence} and \Cref{app:stas}.

\bibliography{bibliography}

\begin{thebibliography}{32}
\providecommand{\natexlab}[1]{#1}
\providecommand{\url}[1]{\texttt{#1}}
\expandafter\ifx\csname urlstyle\endcsname\relax
  \providecommand{\doi}[1]{doi: #1}\else
  \providecommand{\doi}{doi: \begingroup \urlstyle{rm}\Url}\fi

\bibitem[Al~Marjani et~al.(2022)Al~Marjani, Kocak, and
  Garivier]{al2022complexity}
Aymen Al~Marjani, Tomas Kocak, and Aur{\'e}lien Garivier.
\newblock On the complexity of all $\varepsilon$-best arms identification.
\newblock In \emph{Joint European Conference on Machine Learning and Knowledge
  Discovery in Databases}, pp.\  317--332. Springer, 2022.

\bibitem[Aubin(1999)]{aubin1999set}
Jean-Pierre Aubin.
\newblock \emph{Set-Valued Analysis}.
\newblock Birkh{\"a}user Boston, 1999.

\bibitem[Barrier(2023)]{barrier2023contributions}
Antoine Barrier.
\newblock \emph{Contributions to a Theory of Pure Exploration in Sequential
  Statistics}.
\newblock PhD thesis, Ecole normale sup{\'e}rieure de lyon-ENS LYON, 2023.

\bibitem[Barrier et~al.(2022)Barrier, Garivier, and Koc{\'a}k]{barrier2022non}
Antoine Barrier, Aur{\'e}lien Garivier, and Tom{\'a}{\v{s}} Koc{\'a}k.
\newblock A non-asymptotic approach to best-arm identification for gaussian
  bandits.
\newblock In \emph{International Conference on Artificial Intelligence and
  Statistics}, pp.\  10078--10109. PMLR, 2022.

\bibitem[Capp{\'e} et~al.(2013)Capp{\'e}, Garivier, Maillard, Munos, and
  Stoltz]{cappe2013kullback}
Olivier Capp{\'e}, Aur{\'e}lien Garivier, Odalric-Ambrym Maillard, R{\'e}mi
  Munos, and Gilles Stoltz.
\newblock Kullback-leibler upper confidence bounds for optimal sequential
  allocation.
\newblock \emph{The Annals of Statistics}, pp.\  1516--1541, 2013.

\bibitem[Degenne \& Koolen(2019)Degenne and Koolen]{degenne2019pure}
R{\'e}my Degenne and Wouter~M Koolen.
\newblock Pure exploration with multiple correct answers.
\newblock \emph{Advances in Neural Information Processing Systems}, 32, 2019.

\bibitem[Degenne et~al.(2019)Degenne, Koolen, and M{\'e}nard]{degenne2019non}
R{\'e}my Degenne, Wouter~M Koolen, and Pierre M{\'e}nard.
\newblock Non-asymptotic pure exploration by solving games.
\newblock \emph{Advances in Neural Information Processing Systems}, 32, 2019.

\bibitem[Degenne et~al.(2020)Degenne, Shao, and Koolen]{degenne2020structure}
R{\'e}my Degenne, Han Shao, and Wouter Koolen.
\newblock Structure adaptive algorithms for stochastic bandits.
\newblock In \emph{International Conference on Machine Learning}, pp.\
  2443--2452. PMLR, 2020.

\bibitem[Even-Dar et~al.(2002)Even-Dar, Mannor, and Mansour]{even2002pac}
Eyal Even-Dar, Shie Mannor, and Yishay Mansour.
\newblock Pac bounds for multi-armed bandit and markov decision processes.
\newblock In \emph{Computational Learning Theory: 15th Annual Conference on
  Computational Learning Theory, COLT 2002 Sydney, Australia, July 8--10, 2002
  Proceedings 15}, pp.\  255--270. Springer, 2002.

\bibitem[Garivier \& Kaufmann(2016)Garivier and Kaufmann]{garivier2016optimal}
Aur{\'e}lien Garivier and Emilie Kaufmann.
\newblock Optimal best arm identification with fixed confidence.
\newblock In \emph{Conference on Learning Theory}, pp.\  998--1027. PMLR, 2016.

\bibitem[Garivier \& Kaufmann(2021)Garivier and
  Kaufmann]{garivier2021nonasymptoticsequentialtestsoverlapping}
Aurélien Garivier and Emilie Kaufmann.
\newblock Non-asymptotic sequential tests for overlapping hypotheses and
  application to near optimal arm identification in bandit models, 2021.
\newblock URL \url{https://arxiv.org/abs/1905.03495}.

\bibitem[Jin et~al.(2024)Jin, Yang, Tang, Xiao, and Xu]{jin2024optimal}
Tianyuan Jin, Yu~Yang, Jing Tang, Xiaokui Xiao, and Pan Xu.
\newblock Optimal batched best arm identification.
\newblock \emph{Advances in Neural Information Processing Systems},
  37:\penalty0 134947--134980, 2024.

\bibitem[Jourdan \& Degenne(2023)Jourdan and Degenne]{jourdan2023non}
Marc Jourdan and R{\'e}my Degenne.
\newblock Non-asymptotic analysis of a ucb-based top two algorithm.
\newblock \emph{Advances in Neural Information Processing Systems},
  36:\penalty0 68980--69020, 2023.

\bibitem[Jourdan et~al.(2021)Jourdan, Mutn{\`y}, Kirschner, and
  Krause]{jourdan2021efficient}
Marc Jourdan, Mojm{\'\i}r Mutn{\`y}, Johannes Kirschner, and Andreas Krause.
\newblock Efficient pure exploration for combinatorial bandits with semi-bandit
  feedback.
\newblock In \emph{Algorithmic Learning Theory}, pp.\  805--849. PMLR, 2021.

\bibitem[Jourdan et~al.(2022)Jourdan, Degenne, Baudry, de~Heide, and
  Kaufmann]{jourdan2022top}
Marc Jourdan, R{\'e}my Degenne, Dorian Baudry, Rianne de~Heide, and Emilie
  Kaufmann.
\newblock Top two algorithms revisited.
\newblock \emph{Advances in Neural Information Processing Systems},
  35:\penalty0 26791--26803, 2022.

\bibitem[Jourdan et~al.(2023)Jourdan, Degenne, and
  Kaufmann]{jourdan2023varepsilon}
Marc Jourdan, R{\'e}my Degenne, and Emilie Kaufmann.
\newblock An $\epsilon$-best-arm identification algorithm for fixed-confidence
  and beyond.
\newblock \emph{Advances in Neural Information Processing Systems},
  36:\penalty0 16578--16649, 2023.

\bibitem[Juneja \& Krishnasamy(2019)Juneja and Krishnasamy]{juneja2019sample}
Sandeep Juneja and Subhashini Krishnasamy.
\newblock Sample complexity of partition identification using multi-armed
  bandits.
\newblock In \emph{Conference on Learning Theory}, pp.\  1824--1852. PMLR,
  2019.

\bibitem[Kalyanakrishnan et~al.(2012)Kalyanakrishnan, Tewari, Auer, and
  Stone]{kalyanakrishnan2012pac}
Shivaram Kalyanakrishnan, Ambuj Tewari, Peter Auer, and Peter Stone.
\newblock Pac subset selection in stochastic multi-armed bandits.
\newblock In \emph{ICML}, volume~12, pp.\  655--662, 2012.

\bibitem[Kanarios et~al.(2024)Kanarios, Zhang, and Ying]{kanarios2024cost}
Kellen Kanarios, Qining Zhang, and Lei Ying.
\newblock Cost aware best arm identification.
\newblock \emph{arXiv preprint arXiv:2402.16710}, 2024.

\bibitem[Karnin et~al.(2013)Karnin, Koren, and Somekh]{karnin2013almost}
Zohar Karnin, Tomer Koren, and Oren Somekh.
\newblock Almost optimal exploration in multi-armed bandits.
\newblock In \emph{International conference on machine learning}, pp.\
  1238--1246. PMLR, 2013.

\bibitem[Kaufmann \& Koolen(2021)Kaufmann and Koolen]{kaufmann2021mixture}
Emilie Kaufmann and Wouter~M Koolen.
\newblock Mixture martingales revisited with applications to sequential tests
  and confidence intervals.
\newblock \emph{Journal of Machine Learning Research}, 22\penalty0
  (246):\penalty0 1--44, 2021.

\bibitem[Kaufmann et~al.(2016)Kaufmann, Capp{\'e}, and
  Garivier]{kaufmann2016complexity}
Emilie Kaufmann, Olivier Capp{\'e}, and Aur{\'e}lien Garivier.
\newblock On the complexity of best-arm identification in multi-armed bandit
  models.
\newblock \emph{The Journal of Machine Learning Research}, 17\penalty0
  (1):\penalty0 1--42, 2016.

\bibitem[Koc{\'a}k \& Garivier(2020)Koc{\'a}k and Garivier]{kocak2020best}
Tom{\'a}{\v{s}} Koc{\'a}k and Aur{\'e}lien Garivier.
\newblock Best arm identification in spectral bandits.
\newblock \emph{arXiv preprint arXiv:2005.09841}, 2020.

\bibitem[Koc{\'a}k \& Garivier(2021)Koc{\'a}k and Garivier]{kocak2021epsilon}
Tom{\'a}s Koc{\'a}k and Aur{\'e}lien Garivier.
\newblock Epsilon best arm identification in spectral bandits.
\newblock In \emph{IJCAI}, pp.\  2636--2642, 2021.

\bibitem[M{\'e}nard(2019)]{menard2019gradient}
Pierre M{\'e}nard.
\newblock Gradient ascent for active exploration in bandit problems.
\newblock \emph{arXiv preprint arXiv:1905.08165}, 2019.

\bibitem[Moulos(2019)]{moulos2019optimal}
Vrettos Moulos.
\newblock Optimal best markovian arm identification with fixed confidence.
\newblock \emph{Advances in Neural Information Processing Systems}, 32, 2019.

\bibitem[Osogami et~al.(2025)Osogami, Honda, and Komiyama]{osogami2025optimal}
Takayuki Osogami, Junya Honda, and Junpei Komiyama.
\newblock Optimal estimation of the best mean in multi-armed bandits.
\newblock In \emph{The Thirty-ninth Annual Conference on Neural Information
  Processing Systems}, 2025.

\bibitem[Poiani et~al.(2024)Poiani, Jourdan, Kaufmann, and
  Degenne]{poiani2024best}
Riccardo Poiani, Marc Jourdan, Emilie Kaufmann, and R{\'e}my Degenne.
\newblock Best-arm identification in unimodal bandits.
\newblock \emph{arXiv preprint arXiv:2411.01898}, 2024.

\bibitem[Poiani et~al.(2025)Poiani, Bernasconi, and Celli]{poiani2025pure}
Riccardo Poiani, Martino Bernasconi, and Andrea Celli.
\newblock Pure exploration with infinite answers.
\newblock \emph{arXiv preprint arXiv:2505.22473}, 2025.

\bibitem[Russo et~al.(2025)Russo, Song, and Pacchiano]{russo2025pure}
Alessio Russo, Yichen Song, and Aldo Pacchiano.
\newblock Pure exploration with feedback graphs.
\newblock \emph{arXiv preprint arXiv:2503.07824}, 2025.

\bibitem[Simchowitz et~al.(2017)Simchowitz, Jamieson, and
  Recht]{simchowitz2017simulator}
Max Simchowitz, Kevin Jamieson, and Benjamin Recht.
\newblock The simulator: Understanding adaptive sampling in the
  moderate-confidence regime.
\newblock In \emph{Conference on Learning Theory}, pp.\  1794--1834. PMLR,
  2017.

\bibitem[Wang et~al.(2021)Wang, Tzeng, and Proutiere]{wang2021fast}
Po-An Wang, Ruo-Chun Tzeng, and Alexandre Proutiere.
\newblock Fast pure exploration via frank-wolfe.
\newblock \emph{Advances in Neural Information Processing Systems},
  34:\penalty0 5810--5821, 2021.

\end{thebibliography}
\bibliographystyle{iclr2026_conference}

\newpage

\appendix
\etocdepthtag.toc{mtappendix}
\etocsettagdepth{mtchapter}{none}
\etocsettagdepth{mtappendix}{subsection}

\begin{spacing}{1}
\tableofcontents
\end{spacing}

\newpage

\section{Structured Bandits}\label{app:structured}

For completeness, we now show that how to encode known structure such as Lipschitzianity and unimodality through $\mathcal{M}$.

We first focus on the the best-arm identification problem in Lipshitz bandit with finite arms, \ie the same setting studied in \citep{wang2021fast}. This structure can be formalized as follows: $$\mathcal{M} = \{ \bm\mu \in [\mu_{\min}, \mu_{\max}]^K: \exists i ~s.t.~ \mu_i > \mu_k, \land ~\forall k,k', |\mu_k-\mu_{k'}| \le \|  a_k - a_{k'} \|_{\infty}    \},$$ where $l > 0$ is a known constant, $a_k \in \mathbb{R}^d$ are the known feature vectors for the arms, and $[\mu_{\min}, \mu_{\max}]$ are the boundary parameters that we introduced in \Cref{ass:bounded}.

Next, we consider the unimodal setting of \citet{poiani2024best}. 
In this case, we have that:
$$
\mathcal{M} = \{ \bm\mu \in  [\mu_{\min}, \mu_{\max}]^K: \exists i \in [K]: \mu_{i} > \mu_{i+1} \ge \dots \ge \mu_K \land \mu_{i} > \mu_{i -1} \ge \dots \ge \mu_{1} \}.
$$
In other words, each bandit is characterized by an unknown index $i$ such that, the arms' mean will consistently decrease both after and before $i$. 

This reasoning can also be extended to other structures such as the dueling bandit formulation of \citet{wang2021fast}.

\section{Lower Bound}\label{app:lb}

\subsection{Single-Answer Problems}

In this section, we provide a formal statement of the lower bound for single-answer problems.

The following result follows the same arguments of Theorem 1 in \cite{garivier2016optimal}. Since we provide two different expressions for $T^{\star}(\bm\mu)$, we also report a proof for completeness. 

\begin{proposition}[Lower Bound for Single-Answer Problems \cite{garivier2016optimal}]\label{prop:lb-single-answer}
Suppose that $|i^{\star}(\bm\mu)|=1$ for all $\bm\mu \in \Mcal$.
Let $\delta < 0.15$. For any $\bm\mu \in \mathcal{M}$ and any $\delta$-correct algorithm, it holds that $\E_{\bm\mu}[\tau_\delta] \ge T^{\star}(\bm\mu) \log(1/ (2.4 \delta))$.
\end{proposition}
\begin{proof}
    Let $\bm\mu \in \Mcal$ and $\bm\lambda \in \lnot i^{\star}(\bm\mu)$. Then, from change of distribution arguments (\ie Lemma 1 in \cite{kaufmann2016complexity}) and the $\delta$-correctness of the algorithm, we have that:
    \begin{align*}
        \sum_{k \in [K]} \E_{\bm\mu}[N_k(\tau_\delta)] d(\mu_k, \lambda_k) \ge \log(1/(2.4\delta)).
    \end{align*}
    Applying this result for all $\bm\lambda \in \lnot i^{\star}(\bm\mu)$ and since $\bm\mu \notin \cl(\neg i^{\star}(\bm\mu))$, we have that:
    \begin{align*}
        \log(1/(2.4\delta)) & \le \inf_{\bm\lambda \in \neg i^{\star}(\bm\mu)} \sum_{k \in [K]} \E_{\bm\mu}[N_k(\tau_\delta)] d(\mu_k, \lambda_k) \\ 
        & = \E_{\bm\mu}[\tau_\delta] \inf_{\bm\lambda \in \lnot i^{\star}(\bm\mu)} \sum_{k \in [K]} \frac{\E_{\bm\mu}[N_k(\tau_\delta)]}{\E_{\bm\mu}[\tau_\delta]} d(\mu_k, \lambda_k) \\ 
        & \le \E_{\bm\mu}[\tau_\delta] \sup_{\bm\omega \in \Delta_K}\inf_{\bm\lambda \in \lnot i^{\star}(\bm\mu)} \sum_{k \in [K]} \omega_k d(\mu_k, \lambda_k) \\
        & = \E_{\bm\mu}[\tau_\delta] \sup_{\bm\omega \in \Delta_K} \max_{i \in \Ical} \inf_{\bm\lambda \in \lnot i} \sum_{k \in [K]} \omega_k d(\mu_k, \lambda_k),
    \end{align*}
    where, the last step follows from the fact that, for all $i \ne i^{\star}(\bm\mu)$, $\bm\mu \in \neg i$, and, hence $\inf_{\bm\lambda \in \lnot i} \sum_{k \in [K]} \omega_k d(\mu_k, \lambda_k) = 0$.
    The proof then follows by the definition of $T^{\star}(\bm\mu)$ together with the fact that $\sup_{\bm\omega \in \Delta_K}\inf_{\bm\lambda \in \lnot i^{\star}(\bm\mu)} \sum_{k \in [K]} \omega_k d(\mu_k, \lambda_k) > 0$ since $\bm\mu \notin \cl(\bm\mu)$.
\end{proof}

\subsection{Multiple-Answer Problems}

In this section, we formally state the lower bound for multiple answer problems.

\begin{proposition}[Lower Bound for Multiple-Answer Problems \cite{degenne2019pure}]\label{prop:lb-multiple-answers}
    Let $\Ical$ be a finite set and let $\bm\mu \in \Mcal$. Then, it holds that:
    \begin{align}
        \liminf_{\delta \to 0} \frac{\E_{\bm\mu}[\tau_\delta]}{\log(1/\delta)} \ge T^{\star}(\bm\mu),
    \end{align}
    where $T^{\star}(\bm\mu)^{-1}$ is given by:
    \begin{align}
        T^{\star}(\bm\mu)^{-1} & = \sup_{\bm\omega \in \Delta_K} \max_{i \in i^{\star}(\bm\mu)} \inf_{\bm\lambda \in \neg i} \sum_{k \in [K]} \omega_k d(\mu_k, \lambda_k) \label{eq:prop-lb-1} \\
        & = \sup_{\bm\omega \in \Delta_K} \max_{i \in \Ical} \inf_{\bm\lambda \in \neg i} \sum_{k \in [K]} \omega_k d(\mu_k, \lambda_k).\label{eq:prop-lb-2}
    \end{align}
\end{proposition}
\begin{proof}
    The proof is exactly as in Theorem 1 in \cite{degenne2019pure}. Specifically, in that paper, the result was stated with the expression of $T^{\star}(\bm\mu)^{-1}$ given in \Cref{eq:prop-lb-1}. \Cref{eq:prop-lb-2} follows by noticing that, for all $i \notin i^{\star}(\bm\mu)$, then $\bm\mu \in \neg i$, and, hence, $\inf_{\bm\lambda \in \neg i} \sum_{k \in [K]} \omega_k d(\mu_k, \lambda_k) = 0$.
\end{proof}

\section{On the Assumptions}\label{app:assumptions}

In this section, we further discuss our assumptions. As mentioned in the main text, \Cref{ass:subgauss} is a mild requirement that is only needed for concentration arguments. Thus, in the following, we focus on \Cref{ass:bounded}. As our proofs show, \Cref{ass:bounded} is only needed to bound differences of infimum of optimization problems which involve KL divergences. Specifically, it is employed only to control differences in KL for functions of the form:
\begin{align}\label{eq:app-ass}
    \sum_{k}\omega_k ( d(\mu_k, \lambda_{\mu,k}) - d(\mu'_k, \lambda_{\mu',k})),
\end{align}
where $\bm{\lambda}_\mu'$ is the minimizer of $\inf_{\bm\lambda \in \neg i} \omega_k d(\mu'_k, \lambda_k)$ for some $\neg i$. 
Therefore, our results holds for any family $\mathcal{M}$ of bandits for which it is possible to upper bound (in a Lipschitz fashion w.r.t. $\bm\mu$) functions of the form of \Cref{eq:app-ass}. 

At this point, we remark on the following aspects.
\begin{itemize}
    \item \citet{degenne2019non} originally provided the aforementioned intuitive relaxation of \Cref{ass:bounded} and we invite the interested reader to check their Appendix F for further details.
    \item \citet{degenne2019non} also shows that for Gaussian setting on unbounded domains, \Cref{eq:app-ass} can actually be bounded in a Lipschitz fashion. Hence, when dealing with Gaussian distributions, we can actually operate on unbounded domains (\ie we can remove \Cref{ass:bounded}).
\end{itemize} 

Finally, we conclude by noticing that there exist works that have provided finite-confidence guarantees outside of \Cref{ass:bounded}. In particular:
\begin{itemize}
    \item \cite{jourdan2023non} derived finite-confidence results for an optimistic variant of the Top-Two Algorithm without using \Cref{ass:bounded}. Nonetheless, the authors are restricting their analysis to Gaussian distributions, and, as we discussed above, our analysis can easily be generalized to cover this scenario.
    
    \item \cite{barrier2022non} also provides finite-confidence analysis outside of \Cref{ass:bounded}. Nonetheless, their non-asymptotic bounds feature an extra factor$$
\frac{1}{\omega_{\min}(\mu)^2}\,\exp\bigl(-\omega_{\min}(\mu)\bigr),$$
    where $\omega_{min}(\mu) = \min_{k \in [K]} \omega^{\star}_k(\mu)$. We note that $\omega_{\min}$ can be lower-bounded using the minimum gap for Gaussian distributions (see the comment below Theorem 5 in \citet{barrier2022non}), and thus it is not an issue for Gaussian best-arm identification problems, as it can become arbitrary large only for instances which for which the sample complexity lower bound as well tends to $\infty$. Nonetheless, this is not the case for Bernoulli bandits. Indeed, consider a best-arm identification problem in a Bernoulli bandit scenario over $3$ arms. Suppose that $\mu = (x, 0.8, 0.9)$. It is easy to see that $\omega_{1} \to 0$ as $x \to 0$. Therefore, without \cref{ass:bounded}, that finite-confidence guarantees can become vacuous outside of the Gaussian setting.
    
    \item Finally, \citet{wang2021fast} also provided finite-confidence guarantees outside of \Cref{ass:bounded}; nonetheless, additional assumptions are needed in order to obtain the results. We refer the interested reader to Assumption 1-3 in \citet{wang2021fast} for the technical requirements. Here, we only note that their finite-confidence analysis depend on assumptions that involve the gradients of the lower bound as a function of $\bm\omega$. Importantly, the main purpose of their assumptions is the same as ours, i.e., bounding functions of the form of \eqref{eq:app-ass}. This is evident from Lemma 14 in \cite{wang2021fast}. 
\end{itemize}

\section{Non-Asymptotic Bound for Track-and-Stop}\label{app:finite-confidence}

In this section, we analyze the version of \tas that makes use of projection within the sampling rule. Specifically, $\bm\omega(s) \in \bm\omega^{\star}(\tilde{\bm\mu}(t))$, where $\bmMuTilde(t)$ denotes the orthogonal projection of $\bmMuHat(t)$ onto $[\muMin, \muMax]^K$. 
Before delving into the analysis, we note that it holds, due to the convexity of $d(\cdot, \cdot)$ (see, \eg \cite{cappe2013kullback}), we have that:
\begin{align}
    & d(\hat{\mu}_k(t), \lambda) \ge d(\tilde{\mu}_k(t), \lambda), \quad \forall k \in [K], \lambda \in [\muMin, \muMax] \label{eq:proj} \\
    & d(\lambda, \hat{\mu}_k(t)) \ge d(\lambda, \tilde{\mu}_k(t)) \quad \forall k \in [K], \lambda \in [\muMin, \muMax] \label{eq:proj-2}
\end{align}

Now, we start by upper bounding the expectation of $\tau_\delta$ using an arbitrary good-event which implies stopping. The following result is standard in pure exploration works (see, \eg, \cite{degenne2019non}) and the proof is reported for completeness.

\begin{lemma}[Expectation Upper Bound]\label{lemma:expectation}
Consider a sequence of events $\{ \mathcal{E}_t \}_{t\ge 3}$ such that, there exists $T_0(\delta)$ and for $t \ge T_0(\delta)$ it holds that $\mathcal{E}_t \subseteq \{ \tau_\delta \le t \}$. Then, it holds that:
\begin{align*}
    \E_{\bm\mu}[\tau_\delta] \le T_0(\delta) + \sum_{t=3}^{+\infty} \mathbb{P}_{\bm\mu}(\mathcal{E}_t^c).
\end{align*}
\end{lemma}
\begin{proof}
    It holds that:
    \begin{align*}
        \E_{\bm\mu}[\tau_\delta] = \sum_{t=0}^{+\infty} \mathbb{P}_{\bm\mu}(\tau_\delta > t) \le 10K^4 + T_0(\delta) + \sum_{t=3+T_0(\delta)} \mathbb{P}_{\bm\mu}(\tau_\delta > t) \le T_0(\delta) +\sum_{t=1}^{+\infty} \mathbb{P}_{\bm\mu}(\mathcal{E}^c_t).
    \end{align*}
\end{proof}

In our analysis, we will make use of the following good event:
\begin{align*}
    \Ecal_t = \left\{ \forall s \in \left[\lceil \sqrt{t} \rceil, t\right], \sum_{k \in[K]} N_k(s) d(\hat{\mu}_k(s), \mu_k) \le {8K}\log(s) \right\}
\end{align*}
Indeed, it can be shown with probabilistic arguments that $\sum_{t=3}^{+\infty}\mathbb{P}_{\bm\mu}(\Ecal_t^c) \le 2eK$ (\Cref{lemma:good-event}).
In the following, we compact the notation and we define $f(t) \coloneqq 8K \log(t)$. The function $f(t)$ can be understood as an exploration function.

Next, the following lemma is the key result behind our analysis. 

\begin{lemma}[Learning the Equilibrium (\tas)]\label{lemma:p-tas-learning}
    Let $t \ge 10K^4$. If \tas has not stopped at $t$, on $\Ecal_t$, it holds that:
    \begin{align*}
        \frac{\beta_{t,\delta}}{t} \ge \frac{t - \sqrt{t} - 1}{t} T^{\star}(\bm\mu)^{-1} - \sum_{i=1}^4 h_i(t)
    \end{align*}
    where:
    \begin{align*}
        & h_1(t) \le \frac{D\sqrt{2\sigma^2 Kf(t)t}}{{t}}\\
        & h_2(t) \le \frac{LK^2\ln(K) \sqrt{t+K^2}}{t}\\
        & {h}_3(t) \le \frac{D\sqrt{2\sigma^2f(t)}}{t}(K\ln K + 4\sqrt{Kt} + K^2\sqrt{t+K^2})  \\
        & {h}_4(t) \le  \frac{D\sqrt{2\sigma^2f(t)}}{t} \sqrt{8{t}^{3/2} +8Kt\ln(t)}.
    \end{align*}
\end{lemma}
\begin{proof}
    Let us define ${h}_1(t)$ as follows.
    \begin{align*}
    {h}_1(t) \coloneqq \frac{1}{t} \left( \inf_{\bm\lambda \in \lnot i^{\star}(\bm\mu)} \sum_{k \in [K]} N_k(t) d(\mu_k, \lambda_k) - \inf_{\bm\lambda \in \lnot i^{\star}(\bm\mu)} \sum_{k \in [K]} N_k(t) d(\hat{\mu}_k(t), \lambda_k) \right) 
    \end{align*}
    If \tas has not stopped at $t \in \mathbb{N}$, then we have that:
    \begin{align*}
        \frac{\beta_{t,\delta}}{t} & \ge \frac{1}{t} \max_{i \in \Ical} \inf_{\bm\lambda \in \lnot i} \sum_{k \in [K]} N_k(t) d(\hat{\mu}_k(t), \lambda_k) \tag{Stopping Rule}\\ &
        \ge \frac{1}{t} \inf_{\bm\lambda \in \lnot i^{\star}(\bm\mu)} \sum_{k \in [K]} N_k(t) d(\hat{\mu}_k(t), \lambda_k) \\ 
        & \ge \frac{1}{t} \inf_{\bm\lambda \in \lnot i^{\star}(\bm\mu)} \sum_{k \in [K]} N_k(t) d({\mu}_k, \lambda_k) - h_1(t) \tag{Definition of $h_1(t)$}.
     \end{align*}
     Next, we upper bound $h_1(t)$ on the good event $\Ecal_t$.
     \begin{align*}
         h_1(t) & = \frac{1}{t} \left(\inf_{\bm\lambda \in \lnot i^{\star}(\bm\mu)} \sum_{k \in [K]} N_k(t) d(\mu_k, \lambda_k) - \inf_{\bm\lambda \in \lnot i^{\star}(\bm\mu)} \sum_{k \in [K]} N_k(t) d(\hat{\mu}_k(t), \lambda_k) \right) \\
         & \le \frac{1}{t} \sum_{k \in [K]} N_k(t) \sup_{\lambda \in \Mcal} \left( d(\mu_k, \lambda) - d(\hat{\mu}_k(t), \lambda) \right) \\ 
         & \le \frac{1}{t} \sum_{k \in [K]} N_k(t) \sup_{\lambda \in \Mcal} (\nu_{\mu_k} -\nu_{\lambda}) |\mu_k - \hat{\mu}_k(t)| \tag{\Cref{lemma:canonical-exp-family-diff}}  \\
         & \le \frac{D}{t} \sum_{k \in [K]} N_k(t)   |\mu_k - \hat{\mu}_k(t)| \tag{\Cref{ass:bounded}} \\ 
         & \le \frac{D}{t} \sum_{k \in [K]} N_k(t) \sqrt{2 \sigma^2 d(\hat{\mu}_k(t), \mu_k)} \tag{\Cref{ass:subgauss}} \\
         & \le \frac{D}{t} \sum_{k \in [K]} N_k(t) \sqrt{2 \sigma^2 \frac{f(t)}{N_k(t)}} \tag{On $\Ecal_t$, \Cref{lemma:good-event}} \\
         & \le \frac{D\sqrt{2\sigma^2K f(t){t}}}{{t}} \tag{By concavity of $\sqrt{\cdot}$}.
      \end{align*}

    We continue with a lower bound on $\frac{1}{t} \inf_{\bm\lambda \in \lnot i^{\star}(\bm\mu)} \sum_{k \in [K]} N_k(t) d({\mu}_k, \lambda_k)$. Let $\{ \bm\omega(s) \}_{s=1}^t$ be the sequence of empirical oracle weights selected by \tas, \ie $\bm\omega(s) \in \bm\omega^{\star}(\bmMuHat(t))$. Then, we have that: 
    \begin{align*}
          \frac{1}{t} \inf_{\bm\lambda \in \lnot i^{\star}(\bm\mu)} \sum_{k \in [K]} N_k(t) d({\mu}_k, \lambda_k) & \ge  \frac{1}{t} \inf_{\bm\lambda \in \neg i^{\star}(\bm\mu)} \sum_{k \in [K]} \sum_{s=1}^t \omega_k(s) d(\mu_k, \lambda_k) - h_2(t),
    \end{align*}
    where $h_2(t)$ is given by:
    \begin{align*}
          h_2(t) & \coloneqq \frac{1}{t} \inf_{\bm\lambda \in \lnot i^{\star}(\bm\mu)}\sum_{k \in [K]} \left(\sum_{s=1}^t \omega_k(s) -N_k(t) \right) d(\mu_k, \lambda_k) \\ 
          & \le \frac{1}{t} K\ln(K) \sqrt{t+K^2} \inf_{\bm\lambda \in \neg i^{\star}(\bm\mu)} \sum_{k \in [K]} d(\mu_k, \lambda_k) \tag{\Cref{lemma:tracking}} \\
          & \coloneqq \frac{L K^2 \ln(K) \sqrt{t+K^2} }{t} \tag{\Cref{ass:bounded}}.
    \end{align*}

    Then, we lower bound $\frac{1}{t} \inf_{\bm\lambda \in \neg i^{\star}(\bm\mu)} \sum_{k \in [K]} \sum_{s=1}^t \omega_k(s) d(\mu_k, \lambda_k)$. To this end, we recall that, by definition:
    \begin{align*}
        \bm\omega(s) \in \argmax_{\bm\omega \in \Delta_K} \max_{i \in \Ical} \inf_{\bm\lambda \in \neg i} \sum_{k \in [K]} \omega_k d(\tilde{\mu}_k(s), \lambda_k).
    \end{align*}

    Let us denote by $i_s$ an answer that attains the argmax when paired with $\bm\omega(s)$. Then, we have that:
    \begin{align*}
        \frac{1}{t} \inf_{\bm\lambda \in \neg i^{\star}(\bm\mu)} \sum_{k \in [K]} \sum_{s=1}^t \omega_k(s) d(\mu_k, \lambda_k) & \ge \frac{1}{t} \sum_{s=1}^t   \inf_{\bm\lambda \in \neg i^{\star} (\bm\mu)} \sum_{k \in [K]}\omega_k(s) d(\mu_k, \lambda_k)  \\
        & \ge \frac{1}{t} \sum_{s=1}^t   \inf_{\bm\lambda \in \neg i_s} \sum_{k \in [K]}\omega_k(s) d(\mu_k, \lambda_k) \\
        & \ge \frac{1}{t} \sum_{s \ge \sqrt{t}}   \inf_{\bm\lambda \in \neg i_s} \sum_{k \in [K]}\omega_k(s) d(\mu_k, \lambda_k)  \\
        & \ge \frac{1}{t} \sum_{s \ge \sqrt{t}}   \inf_{\bm\lambda \in \neg i_s} \sum_{k \in [K]}\omega_k(s) d(\tilde{\mu}_k(s), \lambda_k) - {h}_3(t),
    \end{align*}
    where the second inequality follows from the fact that (i) if $i_s = i^{\star}(\bm\mu)$ then the claim is trivial, and (ii) if $i_s \ne i^{\star}(\bm\mu)$, then, $\bm\mu \in \neg i_s$ (since $i^{\star}(\bm\mu)$ is single-valued) and $\inf_{\bm\lambda \in \neg i_s} \sum_{k \in [K]}\omega_k(s) d(\mu_k, \lambda_k) = 0$. Finally, the last step follows from the definition of ${h}_3(t)$, that is:
    \begin{align*}
        {h}_3(t) & \coloneqq  \frac{1}{t} \sum_{s \ge \sqrt{t}}   \inf_{\bm\lambda \in \neg i_s} \sum_{k \in [K]}\omega_k(s) d(\tilde{\mu}_k(s), \lambda_k) - \frac{1}{t} \sum_{s \ge \sqrt{t}}   \inf_{\bm\lambda \in \neg i_s} \sum_{k \in [K]}\omega_k(s) d({\mu}_k, \lambda_k) \\ 
        & \le \frac{1}{t} \sum_{s \ge \sqrt{t}} \sum_{k \in [K]} \omega_k(s) \sup_{\bm\lambda \in \Mcal} \left( d(\tilde{\mu}_k(s), \lambda_k) - d(\mu_k, \lambda_k) \right) \\
        & \le \frac{1}{t} \sum_{s \ge \sqrt{t}} \sum_{k \in [K]} \omega_k(s) \sup_{\bm\lambda \in \Mcal} (\nu_{\tilde{\mu}_k(s)} - \nu_{\lambda_k}) |\tilde{\mu}_k(s) - \mu_k|  \tag{\Cref{lemma:canonical-exp-family-diff}} \\
        & \le \frac{D}{t} \sum_{s \ge \sqrt{t}} \sum_{k \in [K]} \omega_k(s) |\tilde{\mu}_k(s) - \mu_k| \tag{\Cref{ass:bounded} and $\bmMuTilde(s) \in [\muMin, \muMax]$} \\
        & \le \frac{D}{t} \sum_{s \ge \sqrt{t}} \sum_{k \in [K]} \omega_k(s) \sqrt{2\sigma^2 d(\tilde{\mu}_k(s), \mu_k)} \tag{\Cref{ass:subgauss}} \\
        & \le \frac{D}{t} \sum_{s \ge \sqrt{t}} \sum_{k \in [K]} \omega_k(s) \sqrt{2\sigma^2 d(\hat{\mu}_k(s), \mu_k)} \tag{\Cref{eq:proj}} \\
        & \le \frac{D \sqrt{2\sigma^2f(t)}}{t} \sum_{s \ge \sqrt{t}} \sum_{k \in [K]} \omega_k(s) \sqrt{\frac{1}{N_k(s)}} \tag{\Cref{lemma:good-event}} \\
        & \le \frac{D \sqrt{2\sigma^2f(t)}}{t} \left( K\ln(K) + 4 \sqrt{Kt} + K^2\sqrt{t+K^2} \right) \tag{\Cref{lemma:tracking}}.
    \end{align*}
    We now have to analyze $\frac{1}{t} \sum_{s \ge \sqrt{t}}   \inf_{\bm\lambda \in \neg i_s} \sum_{k \in [K]}\omega_k(s) d(\tilde{\mu}_k(s), \lambda_k)$.

    Specifically, we have that:
    \begin{align*}
        \frac{1}{t} \sum_{s \ge \sqrt{t}}   \inf_{\bm\lambda \in \neg i_s} \sum_{k \in [K]}\omega_k(s) d(\tilde{\mu}_k(s), \lambda_k) & = \frac{1}{t} \sum_{s \ge \sqrt{t}} \sup_{\bm\omega \in \Delta_K} \max_{j \in [M]} \inf_{\bm\lambda \in \neg j} \sum_{k \in [K]} \omega_k(s) d(\tilde{\mu}_k(s), \lambda_k) \\
        & \ge \frac{1}{t} \sum_{s \ge \sqrt{t}}   \inf_{\bm\lambda \in \neg i^{\star}(\bm\mu)} \sum_{k \in [K]}\omega_k^{\star} d(\tilde{\mu}_k(s), \lambda_k) \\
        & \ge \frac{1}{t} \sum_{s \ge \sqrt{t}}   \inf_{\bm\lambda \in \neg i^{\star}(\bm\mu)} \sum_{k \in [K]}\omega_k^{\star} d({\mu}_k, \lambda_k) - {h}_4(t) \\
        & = \frac{t - \sqrt{t} - 1}{t} T^{\star}(\bm\mu)^{-1} - {h}_4(t) \tag{Def. of $T^{\star}(\bm\mu)^{-1}$},
    \end{align*}
    where the first step follows from the definition of the sampling rule of \tas, in the second one we have chosen any $\bm\omega^{\star} \in \bm\omega^{\star}(\bm\mu)$, and the third one by definition of $h_4(t)$, that is:
    \begin{align*}
        {h}_4(t) & \coloneqq \frac{1}{t} \sum_{s \ge \sqrt{t}}   \inf_{\bm\lambda \in \neg i^{\star}(\bm\mu)} \sum_{k \in [K]}\omega_k^{\star} d({\mu}_k, \lambda_k) - \frac{1}{t} \sum_{s \ge \sqrt{t}}   \inf_{\bm\lambda \in \neg i^{\star}(\bm\mu)} \sum_{k \in [K]}\omega_k^{\star} d(\tilde{\mu}_k(s), \lambda_k)
    \end{align*}
    Now, we conclude the proof by giving an upper bound on ${h}_4(t)$.
    \begin{align*}
        {h}_4(t) & \le \frac{1}{t} \sum_{s \ge \sqrt{t}} \sum_{k \in K} \omega_k^{\star} \sup_{\bm\lambda \in \Mcal} \left( d(\mu_k, \lambda_k) - d(\tilde{\mu}_k(s), \lambda_k) \right) \\ 
        & \le \frac{1}{t} \sum_{s \ge \sqrt{t}} \sum_{k \in K} \omega_k^{\star} \sup_{\bm\lambda \in \Mcal} (\nu_{\mu_k} - \nu_{\lambda_k}) |\mu_k - \tilde{\mu}_k(s)| \tag{\Cref{lemma:canonical-exp-family-diff}} \\
        & \le \frac{D}{t} \sum_{s \ge \sqrt{t}} \sum_{k \in K} \omega_k^{\star} |\mu_k - \tilde{\mu}_k(s)| \tag{\Cref{ass:bounded}} \\
        & \le \frac{D}{t} \sum_{s \ge \sqrt{t}} \| \bm\mu -\bmMuTilde(s) \|_{\infty} \\ 
        & \le \frac{D}{t} \sum_{s \ge \sqrt{t}} \max_{k \in [K]} \sqrt{2\sigma^2 d(\tilde{\mu}_k(s), \mu_k) } \tag{\Cref{ass:subgauss}} \\
        & \le \frac{D \sqrt{2\sigma^2 f(t)}}{t} \sum_{s \ge \sqrt{t}} \max_{k \in [K]} \sqrt{\frac{1}{N_k(s)}} \tag{\Cref{lemma:good-event}} \\
        & \le \frac{D \sqrt{2\sigma^2 f(t)}}{t} \sum_{s \ge \sqrt{t}} \sqrt{\frac{1}{\sqrt{s+K^2}-2K}} \tag{\Cref{lemma:tracking} and $t \ge 10K^4$} \\
        & \le \frac{D \sqrt{2\sigma^2 f(t)}}{t}  \sqrt{ t \sum_{s \ge \sqrt{t}}  {\frac{1}{\sqrt{s+K^2}-2K}}}  \tag{Concavity of $\sqrt{\cdot}$ and $t \ge 10K^4$} \\
        & \le \frac{D \sqrt{2\sigma^2 f(t)}}{t} \sqrt{8{t}^{3/2} +8Kt\ln(t)} \tag{Integral test and algebraic manipulations}
    \end{align*}
    which concludes the proof.\footnote{The requirement of $t \ge 10K^4$ is essentially needed to guarantee that the denominators in those steps are always positive.} 
\end{proof}

\begin{proof}[Proof of \Cref{theo:proj-tas-finite-confidence}]
    Let $t \ge 10K^4$. Then, for $t \ge 10K^4 + T_0(\delta)$, by \Cref{lemma:p-tas-learning}, we have that $\Ecal_t \subseteq \{ \tau_\delta \le t \}$. \Cref{lemma:expectation} and \Cref{lemma:good-event}, then conclude the proof. To this end, it is sufficient to note that $T^{\star}(\bm\mu)\sum_{i=1}^4 h(t) t + \sqrt{t} + 1\le g(t)$.\footnote{In the relevant regime where $D, L, \sigma, T^{\star}(\bm\mu) \ge 1$.} Indeed, using $t \ge 10^4$ and simple algebraic arguments, we have that:
    \begin{align*}
        & t h_1(t) \le 4\sigma D \sqrt{Kt\log(t)} \le 4\sigma D L K^2 \log(K) \sqrt{t \log^2(t)} \\
        & t h_2(t) \le LK^2 \log(K) \sqrt{2t} \le DL K^2 \log(K) \sqrt{2t\log^2(t)} \\
        & t h_3(t) \le 4 \sigma D \sqrt{\log^2(t)} \left( K \log K + 4 \sqrt{Kt} + K^2 \sqrt{2t} \right) \le 4 \sigma D L \sqrt{\log^2(t)} \left( 10 K^2 \log(K) \sqrt{t}  \right) \\
        & t h_4(t) \le 16\sigma D \sqrt{Kt^{3/2} \log(t)} + 4\sigma D L K^2 \log(K) \sqrt{8t \log^2(t)}.
    \end{align*}
    Combining these inequalities, we obtain:
    \begin{align*}
        \sum_{i=1}^t th_i(t) \le 62 \sigma DLK^2 \log(K) \sqrt{t \log^2(t)} + 16 \sigma D\sqrt{Kt^{3/2}\log(t)}.
    \end{align*}
    Thus, since from \Cref{lemma:p-tas-learning} we know that, for $t$ such that:
    \begin{align*}
        T^{\star}(\bm\mu) \beta_{t,\delta} + \sqrt{t} + 1 + T^{\star}(\bm\mu) \sum_{i=1}^{4} th_i(t) \le t,
    \end{align*}
    implies stopping on the good event, we also have that, \tas is guaranteed to stop whenever:
    \begin{align*}
        T^{\star}(\bm\mu) \beta_{t,\delta} + T^{\star}(\bm\mu)g(t) \le t.
    \end{align*}
    Rearranging the terms give the desired expression of $T_0(\delta)$.
\end{proof}

\section{A Simple Fix without Using Projections}\label{app:finite-confidence-proj-tas}

In this section, we discuss what happens when \tas is not using projections in the sampling rule. The key idea is that there exists a time $T_{\Mcal}$ such that, for subsequent steps $t$, then the empirical mean always lies within the interval $[\muMin, \muMax]$. Before that, we make a remark on \Cref{ass:bounded}. Specifically, fix any $\bm\mu \in \Mcal$ and let $F_k \coloneqq \min\{ |\mu_k - \muMin|, |\mu_k - \muMax| \}$ and $F = \min_{k \in [K]} F_k$. As we discussed in \Cref{sec:finite-confidence}, $F > 0$ holds to the fact that $\Theta$ is an open interval and $[\muMin, \muMax]$ is closed and contained in $\Theta$.

The following lemma shows the existence of such a $T_{\Mcal}$.

\begin{lemma}[Empirical Means Lies in a Good Region]\label{lemma:mu-hat-good-region}
    Under \Cref{ass:subgauss} and \Cref{ass:bounded}, there exists a time $T_{\Mcal} \in \Naturals$ such that, for all $t \ge T_{\Mcal}$, on $\Ecal_t$, it holds that $\bmMuHat(s) \in [\muMin, \muMax]$ for all $s \ge \sqrt{t}$. Specifically,
    \begin{align*}
        T_{\Mcal} = \max \left\{  10K^4 , \inf \left\{n \in \Naturals: \sqrt{\frac{64K \sigma^2 \log(n)}{\sqrt{\sqrt{n}+K^2}-2K}} \le F  \right\}  \right\}
    \end{align*}
\end{lemma}
\begin{proof}
    Let $\bar{T}$ be such that, for all $t \ge \bar{T}$, $\sqrt{\sqrt{t}+K^2}-2K$, \ie $\bar{T} \ge 10K^4$. Then, let $t \ge \bar{T}$.

    Let $F_k \coloneqq \min\{ |\mu_k - \muMin|, |\mu_k - \muMax|   \}$ and $F = \min_{k \in [K]} F_k$. Then, we have that if $\|\bmMuHat(t) - \bm\mu \|_{\infty} \le F$, it holds that $\bmMuHat(t) \in [\muMin, \muMax]$. As discussed above, from \Cref{ass:bounded}, $F > 0$. 
    
    Now, on $\Ecal_t$, for any $s \ge \sqrt{t}$, we have that:
    \begin{align*}
        \| \bm\mu - \bmMuHat(s) \|_{\infty} & \le \max_{k \in [K]} \sqrt{2\sigma^2 d(\hat{\mu}_k(s), \mu_k)}  \tag{\Cref{ass:subgauss}}\\ 
        & \le   \max_{k \in [K]}  \sqrt{\frac{2\sigma^2f(s)}{N_k(s)}} \tag{\Cref{lemma:good-event}} \\ 
        & \le \sqrt{\frac{2\sigma^2 f(s)}{\sqrt{s+K^2}-2K}} \tag{\Cref{lemma:tracking}} \\
        & \le \sqrt{\frac{4\sigma^2 f(t)}{\sqrt{\sqrt{t}+K^2}-2K}} \tag{$s \ge \sqrt{t}$ and $t \ge \bar{T}$}.
    \end{align*}

    Then, letting $T_{\Mcal} = \max \left\{  \bar{T}, \inf \left\{n \in \Naturals: \sqrt{\frac{4\sigma^2 f(n)}{\sqrt{\sqrt{n}+K^2}-2K}} \le F  \right\}  \right\}$ concludes the proof.
\end{proof}

Then, one can exploit \Cref{lemma:mu-hat-good-region} to obtain a result that is analogous to one of \Cref{theo:proj-tas-finite-confidence}, just adding $T_{\Mcal}$ to the finite-confidence upper bound on $\E_{\bm\mu}[\tau_\delta]$. Indeed, \Cref{lemma:p-tas-learning} holds as-is by analyzing any $t$ such that $t \ge T_{\Mcal}$.

\section{Non-Asymptotic Bound For Sticky Track-and-Stop}\label{app:stas}

In this section, we derive finite-confidence bounds for Sticky Track-and-Stop. We start with the following result, which shows the existence of a finite time after which (under the good event) the answer $i_t$ chosen by \stas follows within a \quotes{good set}, \ie $i_F(\bm\mu) \cup (\Ical \setminus i^{\star}(\bm\mu))$.

\begin{lemma}[Good Answers on the Good Event]\label{lemma:good-answers}
    Let $T_{\bm\mu}$ be defined as follows
    \begin{align*}
    T = \max \left\{  10K^4  , \inf \left\{ n \in \Naturals, \sqrt{\frac{64K \sigma^2 \log(n)}{\sqrt{\sqrt{n}+K^2}-2K}} \le \epsilon_{\bm\mu} \right\} \right\}, 
    \end{align*}
    where $\epsilon_{\bm\mu} > 0$ is a problem dependent constant. 
    Then, for all $t \ge {T}$, on $\Ecal_t$, it holds that $i_s \in i_F(\bm\mu) \cup (\Ical \setminus i^{\star}(\bm\mu))$ for all $s \ge \sqrt{t}$.
\end{lemma}
\begin{proof}
    We recall that $\bm\mu \mapstoto i_F(\bm\mu)$ is upper hemicontinuous (Theorem 4 in \cite{degenne2019pure}). This implies that there exists $\epsilon_{\bm\mu} > 0$ such that, for all $\bm\mu': \|\bm\mu - \bm\mu' \|_{\infty} \le \epsilon_{\bm\mu}$, it holds that $i_F(\bm\mu') \subseteq i_F(\bm\mu) \cup \left( \Ical \setminus i^{\star}(\bm\mu)\right)$.

    Now, consider $\bar{T}$ defined as follows:
    \begin{align*}
        \bar{T} = \inf \left\{n \in \Naturals: \sqrt{ \sqrt{n} + K^2 } - 2K > 0 \right\},
    \end{align*}
    that is $\bar{T} =  10K^4 $.
    Then, for all $t \ge \bar{T}$ and all $s \ge \sqrt{t}$, it holds that $\sqrt{ s + K^2 } - 2K > 0$.

    Consider $t \ge \bar{T}$, and let us introduce, for all $\bm\mu, \bm\mu' \in \Mcal$, $\textup{ch}(\bm\mu, \bm\mu') = \inf_{\bm\lambda \in \Reals^K} \sum_{k \in [K]} (d(\lambda_k, \mu_k) + d(\lambda_k, \mu_k'))$. Now, on $\Ecal_t$ and for $s \ge \sqrt{t}$, we have that:
    \begin{align*}
        \sum_{k \in [K]} N_k(s) d(\hat{\mu}_k(s), \mu_k)) \le 8K \log(s).
    \end{align*}
    Furthermore, by definition, for all $\bm\mu' \in C_s$, we also have that:
        \begin{align*}
        \sum_{k \in [K]} N_k(s) d(\hat{\mu}_k(s), \mu_k')) \le 8K \log(s).
    \end{align*}
    As a consequence, by applying \Cref{lemma:tracking}, it holds that:
    \begin{align*}
        \textup{ch}(\bm\mu, \bm\mu') \left( \sqrt{s+K^2} - 2K \right) \le \sum_{k \in } N_k(s) \left( d(\hat{\mu}_k(s), \mu_k) + d(\hat{\mu}_k(s), \mu'_k \right) \le 16K\log(s). 
    \end{align*}
    For $t \ge \bar{T}$, and using the definition of $\textup{ch}$, this leads to:\footnote{The lower bound on $\textup{ch}(\cdot, \cdot)$ follows from using the sub-gaussianity of the arms to lower bound the divergences $d$ with the difference in means and solving the resulting $\inf$ problem over $\Reals^K$.}
    \begin{align*}
        \frac{\|\bm\mu - \bm\mu' \|_{\infty}^2}{8\sigma^2} \le \textup{ch}(\bm\mu, \bm\mu') \le \frac{16K\log(s)}{\sqrt{s+K^2}-2K}.
    \end{align*}
    which leads to:
    \begin{align*}
        \|\bm\mu - \bm\mu' \|_{\infty} \le \sqrt{\frac{32K \sigma^2 \log(s)}{\sqrt{s+K^2}-2K}}, \quad \textup{ on } \Ecal_t~ \forall s \ge \sqrt{t}, \bm\mu' \in C_s.
    \end{align*}
    
    Thus, for $t \ge \max \left\{ \bar{T}, \inf \left\{ n \in \Naturals, \sqrt{\frac{64K \sigma^2 \log(n)}{\sqrt{\sqrt{n}+K^2}-2K}} \le \epsilon_{\bm\mu} \right\} \right\}$, it holds that:
    \begin{align*}
        \| \bm\mu - \bm\mu'\|_{\infty} \le \epsilon_{\bm\mu}, \quad \textup{ on } \Ecal_t~\forall s \ge \sqrt{t}, \bm\mu' \in C_s.
    \end{align*}
    Now, since $i_s \in \Ical_s = \bigcup_{\bm\mu' \in C_s}i_F(\bm\mu')$ and $\| \bm\mu - \bm\mu'\|_{\infty} \le \epsilon_{\bm\mu}$ for all $\bm\mu' \in C_s$, it follows (by definition of $\epsilon_{\bm\mu}$) that, on $\Ecal_t$, for $s \ge \sqrt{t}$, $i_s \in i_F(\bm\mu) \cup (\Ical \setminus i^{\star}(\bm\mu))$, thus concluding the proof. 
\end{proof}

Next, the following result is the key lemma that provides a lower bound, under the good event $\Ecal_t$, on the information gathered by \stas.  

\begin{lemma}[Learning the Equilibrium (\stas)]\label{lemma:eq-learning-stas}
    Let $t \ge  10K^4  $ and let $T_{\bm\mu}$ as in \Cref{lemma:good-answers}. Define $\widetilde{T} = \max\{ T_{\bm\mu}, \lceil \sqrt{t} \rceil \}$. Then, for \stas, on $\Ecal_t$, it holds that:
    \begin{align*}
        \frac{\beta_{t,\delta}}{t} \ge \frac{t - \widetilde{T}}{t} T^{\star}(\bm\mu)^{-1} - \sum_{i=1}^5 h_i(t),
    \end{align*}
    where
    \begin{align*}
        & h_1(t) \le \frac{D \sqrt{2\sigma^2Kf(t)t}}{t} \\
        & h_2(t) \le \frac{LK^2 \ln(K) \sqrt{t+K^2}}{t} \\
        & h_3(t) \le \frac{D\sqrt{2\sigma^2f(t)}}{t} \left( K\ln(K) + 4\sqrt{Kt} + K^2\sqrt{t+K^2} \right) \\
        & h_4(t) \le \frac{D\sqrt{2\sigma^2 f(t)}}{t} \sqrt{8t^{3/2}+8Kt\ln(t)} \\
        & h_5(t) \le \frac{2D\sqrt{2\sigma^2 f(t)}}{t} \sqrt{8t^{3/2}+8Kt\ln(t)}.
    \end{align*}
    
\end{lemma}
\begin{proof}
    Let $h_1(t)$ be defined as follows:
    \begin{align*}
        h_1(t) = \frac{1}{t} \max_{i \in \Ical} \inf_{\bm\lambda \in \neg i} \sum_{k \in [K]} N_k(t) d(\mu_k, \lambda_k) - \frac{1}{t} \max_{i \in \Ical} \inf_{\bm\lambda \in \neg i} \sum_{k \in [K]} N_k(t) d(\hat{\mu}_k(t), \lambda_k)
    \end{align*}

    If Sticky Track-and-Stop has not stopped at $t \in \Naturals$, then, it holds that:
    \begin{align*}
        \frac{\beta_{t,\delta}}{t} & \ge \frac{1}{t} \max_{i \in \Ical} \inf_{\bm\lambda \in \neg i} \sum_{k \in [K]} N_k(t) d(\hat{\mu}_k(t), \lambda_k) \tag{Stopping Rule}\\
        & = \frac{1}{t} \max_{i \in \Ical}\inf_{\bm\lambda \in \neg i} \sum_{k \in [K]} N_k(t) d({\mu}_k(t), \lambda_k) - h_1(t) \tag{Definition of $h_1(t)$}. \\
    \end{align*}
    Next, we upper bound $h_1(t)$ under the good event $\Ecal_t$.
    \begin{align*}
        h_1(t) & \le \frac{1}{t} \sum_{k \in [K]} N_k(t) \sup_{\bm\lambda \in \Mcal} (d(\mu_k, \lambda_k) - d(\hat{\mu}_k(t), \lambda_k)) \\
        & \le \frac{1}{t} \sum_{k \in [K]} N_k(t) \sup_{\bm\lambda \in \Mcal} (d(\mu_k, \lambda_k) - d(\hat{\mu}_k(t), \lambda_k)) \\
        & \le \frac{1}{t} \sum_{k \in [K]} N_k(t) \sup_{\bm\lambda \in \Mcal} (\nu_{\mu_k} - \nu_{\lambda_k}) |\mu_k - \hat{\mu}_k(t)| \tag{\Cref{lemma:canonical-exp-family-diff}} \\
        & \le \frac{D}{t} \sum_{k \in [K]} N_k(t) |\mu_k - \hat{\mu}_k(t)| \tag{\Cref{ass:bounded}} \\
        & \le \frac{D}{t} \sum_{k \in [K]} N_k(t) \sqrt{2\sigma^2d(\hat{\mu}_k(t), \mu_k)} \tag{\Cref{ass:subgauss}} \\
        & \le \frac{D}{t} \sum_{k \in [K]} N_k(t) \sqrt{2\sigma^2 \frac{f(t)}{N_k(t)}} \tag{\Cref{lemma:good-event}} \\
        & \le \frac{D\sqrt{2\sigma^2Kf(t)t}}{t} \tag{Concavity of $\sqrt{\cdot}$}.
    \end{align*}

    Next, we continue by analyzing $\frac{1}{t} \max_{i \in \Ical}\inf_{\bm\lambda \in \neg i} \sum_{k \in [K]} N_k(t) d({\mu}_k(t), \lambda_k)$. Let $\{ \bm\omega(s) \}_s$ be the sequence of empirical oracle weights computed by Sticky Track-and-Stop. To this end, let $h_2(t)$ be defined as follows:
    \begin{align*}
        h_2(t) = \frac{1}{t} \max_{i \in \Ical} \inf_{\bm\lambda \in \neg i} \sum_{k \in [K]} \left(\sum_{s=1}^t \omega_k(s) - N_k(t)\right) d(\mu_k, \lambda_k).
    \end{align*}
    Then, by definition of $h_2(t)$, we have that:
    \begin{align*}
        \frac{1}{t} \max_{i \in \Ical}\inf_{\bm\lambda \in \neg i} \sum_{k \in [K]} N_k(t) d({\mu}_k(t), \lambda_k) & \ge \frac{1}{t} \max_{i \in \Ical} \inf_{\bm\lambda \in \neg i} \sum_{s=1}^t \sum_{k \in [K]} \omega_k(s) d(\mu_k, \lambda_k) - h_2(t)
    \end{align*}
    Next, we upper bound $h_2(t)$.
    \begin{align*}
        h_2(t) & \le \frac{1}{t} \max_{i \in \Ical} \inf_{\bm\lambda \in \neg i} \sum_{k \in [K]} K \ln(K)\sqrt{t+K^2} d(\mu_k, \lambda_k) \tag{\Cref{lemma:tracking}} \\
        &  \le \frac{LK^2 \ln(K) \sqrt{t+K^2}}{t} \tag{\Cref{ass:bounded}}.
    \end{align*}
    Next, we focus on $\frac{1}{t} \max_{i \in \Ical} \inf_{\bm\lambda \in \neg i}\sum_{s=1}^t \sum_{k \in [K]} \omega_k(s) d(\mu_k, \lambda_k)$. Let $\imath \in \Ical$ be such that, given the subset of answers $i_F(\bm\mu)$, then, the pre-specified total order over $\Ical$ selects $\imath$. Furthermore, let ${T}_{\bm\mu}$ be as in \Cref{lemma:good-answers} and let $\widetilde{T} = \max\{ {T}_{\bm\mu}, \lceil \sqrt{t} \rceil \}$ Then, it holds that:
    \begin{align*}
        \frac{1}{t} \max_{i \in \Ical} \inf_{\bm\lambda \in \neg i} \sum_{s=1}^t \sum_{k \in [K]} \omega_k(s) d(\mu_k, \lambda_k) & \ge \frac{1}{t}  \sum_{s=1}^t \inf_{\bm\lambda \in \neg \imath} \sum_{k \in [K]} \omega_k(s) d(\mu_k, \lambda_k)  \\
        & \ge \frac{1}{t} \sum_{s=\widetilde{T}}^t \inf_{\bm\lambda \in \neg \imath} \sum_{k \in [K]} \omega_k(s) d(\mu_k, \lambda_k) \\
        & \ge \frac{1}{t} \sum_{s=\widetilde{T}}^t \inf_{\bm\lambda \in \neg i_s} \omega_k(s) d(\mu_k, \lambda_k),
    \end{align*}
    where in the last step, we have used the fact that, (i) if $i_s \in i_F(\bm\mu)$, then $i_s = \imath$ on the good event $\Ecal_t$\footnote{Indeed, in that case, $\bm\mu \in C_s$, and, consequently, $i_F(\bm\mu) \in \Ical_s$. Since the algorithm selects answers according to a total order, it cannot select any answer in $i_F(\bm\mu)$ which is not $\imath$.} and (ii) if $i_s \notin i_F(\bm\mu)$, then $i_s \notin i^{\star}(\bm\mu)$ due to the definition of $\bar{T}$. Then, in this case, we have that $\bm\mu \in \neg i_s$ and $\inf_{\bm\lambda \in \neg i_s} \omega_k(s) d(\mu_k, \lambda_k) = 0$. Next, we have that:
    \begin{align*}
        \frac{1}{t} \sum_{s=\widetilde{T}}^t \inf_{\bm\lambda \in \neg i_s} \sum_{k \in [K]} \omega_k(s) d(\mu_k, \lambda_k) & \ge \frac{1}{t} \sum_{s=\widetilde{T}}^t \inf_{\bm\lambda \in \neg i_s} \sum_{k \in [K]} \omega_k(s) d(\tilde{\mu}_k(s), \lambda_k) - h_3(t),
    \end{align*}
    where $h_3(t)$ is given by:
    \begin{align*}
        h_3(t) & \coloneqq \frac{1}{t} \sum_{s=\widetilde{T}}^t \inf_{\bm\lambda \in \neg i_s} \sum_{k \in [K]} \omega_k(s)  d(\tilde{\mu}_k(s), \lambda_k) - \frac{1}{t} \sum_{s=\widetilde{T}}^t \inf_{\bm\lambda \in \neg i_s} \sum_{k \in [K]} \omega_k(s) d(\mu_k, \lambda_k) 
    \end{align*}
    Now, we have that:
    \begin{align*}
        h_3(t) & \le \frac{1}{t} \sum_{s = \widetilde{T}}^t \sum_{k \in [K]} \omega_k(s) \sup_{\bm\lambda \in \Mcal} \left( d(\tilde{\mu}_k(s), \lambda_k) - d(\mu_k, \lambda_k) \right) \\
        & \le \frac{D}{t} \sum_{s = \widetilde{T}} \sum_{k \in [K]} \omega_k(s)  |\hat{\mu}_k(s) - {\mu}_k|  \tag{\Cref{lemma:canonical-exp-family-diff}, \Cref{ass:bounded}, def. of $\tilde{\mu}_k(s)$} \\
        & \le \frac{D}{t} \sum_{s = \widetilde{T}}^t \sum_{k \in [K]} \omega_k(s) \sqrt{2\sigma^2 d(\hat{\mu}_k(s), \mu_k)} \tag{\Cref{ass:subgauss}} \\
        & \le \frac{D \sqrt{2 \sigma^2 f(t)}}{t} \sum_{s = 1}^t \sum_{k \in [K]} \omega_k(s) \sqrt{\frac{1}{N_k(s)}} \tag{\Cref{lemma:good-event}} \\
        & \le \frac{D \sqrt{2 \sigma^2 f(t)}}{t} \left( K \ln(K) + 4 \sqrt{Kt}+ K^2\sqrt{t+K^2} \right) \tag{\Cref{lemma:tracking}}.
    \end{align*}

    Now, we continue by analyzing 
    \begin{align*}
    \frac{1}{t} \sum_{s=\widetilde{T}}^t \inf_{\bm\lambda \in \neg i_s} \sum_{k \in [K]} \omega_k(s) d(\tilde{\mu}_k(s), \lambda_k).    
    \end{align*}
    Specifically, by definition of $i_s$ and $\bm\omega(s)$, we have that:
    \begin{align*}
        \frac{1}{t} \sum_{s=\widetilde{T}}^t \inf_{\bm\lambda \in \neg i_s} \sum_{k \in [K]} \omega_k(s) d(\tilde{\mu}_k(s), \lambda_k) & = \frac{1}{t} \sum_{s = \widetilde{T}}^t \max_{\bm\omega \in \Delta_K} \inf_{\bm\lambda \in \neg i_s} \sum_{k \in [K]} \omega_k d(\tilde{\mu}_k(s), \lambda_k) \\ 
        & = \frac{1}{t} \sum_{s = \widetilde{T}}^t \max_{\bm\omega \in \Delta_K} \inf_{\bm\lambda \in \neg i_s} \sum_{k \in [K]} \omega_k d(\mu'_k(s), \lambda_k) - h_4(t),
    \end{align*}
    where $\bm\mu'(s) \in \Mcal$ is such that $i_s \in i_F(\bm\mu'(s))$ and $h_4(t)$ is given by:
    \begin{align*}
        h_4(t) & \coloneqq \frac{1}{t} \sum_{s = \widetilde{T}}^t \max_{\bm\omega \in \Delta_K} \inf_{\bm\lambda \in \neg i_s} \sum_{k \in [K]} \omega_k d(\mu'_k(s), \lambda_k)  - \frac{1}{t} \sum_{s = \widetilde{T}}^t \max_{\bm\omega \in \Delta_K} \inf_{\bm\lambda \in \neg i_s} \sum_{k \in [K]} \omega_k d(\tilde{\mu}_k(s), \lambda_k).
    \end{align*}
    Now, we upper bound $h_4(t)$.
    \begin{align*}
        h_4(t) &
        \le \frac{1}{t} \sum_{s = \widetilde{T}}^t \max_{\bm\omega \in \Delta_k} \left( \inf_{\bm\lambda \in \neg i_s} \sum_{k \in [K]} \omega_k d(\mu'_k(s), \lambda_k) - \inf_{\bm\lambda \in \neg i_s} \sum_{k \in [K]} \omega_k d(\tilde{\mu}_k(s), \lambda_k) \right) \\
        & \le \frac{1}{t} \sum_{s = \widetilde{T}}^t \max_{\bm\omega \in \Delta_k} \left( \sum_{k \in [K]} \omega_k \sup_{\bm\lambda \in \Mcal} \left( d(\mu'_k(s), \lambda_k) - d(\tilde{\mu}_k(s), \lambda_k)) \right) \right) \\
        & \le \frac{1}{t} \sum_{s = \widetilde{T}}^t \max_{\bm\omega \in \Delta_K} \sum_{k \in [K]} \omega_k \sup_{\bm\lambda \in \Mcal} (\nu_{\mu'_k(s)} - \nu_{\lambda_k}) |\mu'_k(s) - \tilde{\mu}_k(s)| \tag{\Cref{lemma:canonical-exp-family-diff}} \\
        & \le  \frac{D}{t} \sum_{s = \widetilde{T}}^t \max_{\bm\omega \in \Delta_K} \sum_{k \in [K]} \omega_k |\mu'_k(s) - \hat{\mu}_k(s)| \tag{\Cref{ass:bounded}} \\
        & \le \frac{D}{t} \sum_{s = \widetilde{T}}^t \max_{\bm\omega \in \Delta_K} \sum_{k \in [K]} \omega_k \sqrt{2 \sigma^2 d(\hat{\mu}_k(s), \mu'_k(s)) } \tag{\Cref{ass:subgauss}} \\
        & \le \frac{D \sqrt{2 \sigma^2 f(t)}}{t} \sum_{s = \widetilde{T}}^t \max_{\bm\omega \in \Delta_K} \sum_{k \in [K]} \omega_k \sqrt{\frac{1}{N_k(s)}} \tag{$\mu'_k(s) \in C_s$} \\
        & \le \frac{D \sqrt{2 \sigma^2 f(t)}}{t} \sum_{s = \widetilde{T}}^t  \sqrt{\frac{1}{\sqrt{s+K^2} - 2K}} \tag{\Cref{lemma:tracking}} \\
        & \le \frac{D \sqrt{2 \sigma^2 f(t)}}{t} \sqrt{ t \sum_{s = \widetilde{T}}^t  {\frac{1}{\sqrt{s+K^2} - 2K}} } \tag{Concavity of $\sqrt{\cdot}$} \\
        & \le \frac{D \sqrt{2 \sigma^2 f(t)}}{t} \sqrt{8t^{3/2} + 8Kt\ln(t)} \tag{Integral test and algebraic manipulations}
    \end{align*}

    Next, it remains to analyze $\frac{1}{t} \sum_{s = \widetilde{T}}^t \max_{\bm\omega \in \Delta_K} \inf_{\bm\lambda \in \neg i_s} \sum_{k \in [K]} \omega_k d(\mu'_k(s), \lambda_k)$. Let $\bm\omega^{\star} \in \bm\omega^{\star}(\bm\mu, \neg \imath)$. Then, we have that:
    \begin{align*}
        \frac{1}{t} \sum_{s = \widetilde{T}}^t \max_{\bm\omega \in \Delta_K} \inf_{\bm\lambda \in \neg i_s} \sum_{k \in [K]} \omega_k d(\mu'_k(s), \lambda_k) & = \frac{1}{t} \sum_{s = \widetilde{T}}^{t} \max_{i \in \Ical} \max_{\bm\omega \in \Delta_K} \inf_{\bm\lambda \in \neg i} \sum_{k \in [K]} \omega_k d(\mu'_k(s), \lambda_k) \\
        & \ge \frac{1}{t} \sum_{s = \widetilde{T}}^t \inf_{\bm\lambda \in \neg \imath} \sum_{k \in [K]} \omega_k^{\star} d(\mu'_k(s), \lambda_k) \\
        & \ge \frac{1}{t} \sum_{s = \widetilde{T}}^t \inf_{\bm\lambda \in \neg \imath} \sum_{k \in [K]} \omega_k^{\star} d(\mu_k, \lambda_k) - h_5(t) \\
        & = \frac{t - \widetilde{T}}{t} T^{\star}(\bm\mu)^{-1} -h_5(t),
    \end{align*}
    where the first step follows from the fact that $i_s \in i_F(\bm\mu'(s))$ and the last third one from the definition of $h_5(t)$, that is:
    \begin{align*}
        h_5(t) & \coloneqq \frac{1}{t} \sum_{s = \widetilde{T}}^t \inf_{\bm\lambda \in \neg \imath} \sum_{k \in [K]} \omega_k^{\star} d(\mu_k, \lambda_k)  - \frac{1}{t} \sum_{s = \widetilde{T}}^t \inf_{\bm\lambda \in \neg \imath} \sum_{k \in [K]} \omega_k^{\star} d(\mu'_k(s), \lambda_k) \\ 
        & \le \frac{1}{t } \sum_{s= \widetilde{T}}^t \sum_{k \in [K]} \omega_k^{\star} \sup_{\bm\lambda \in \Mcal} ( d(\mu_k, \lambda_k)- d(\mu'_k(s), \lambda_k)) \\
        & \le \frac{1}{t } \sum_{s= \widetilde{T}}^t \sum_{k \in [K]} \omega_k^{\star} \sup_{\bm\lambda \in \Mcal} (\nu_{\mu_k} - \nu_{\lambda_k}) |\mu_k - \mu'_k(s)|  \tag{\Cref{lemma:canonical-exp-family-diff}} \\
        & \le \frac{D}{t } \sum_{s= \widetilde{T}}^t \sum_{k \in [K]} \omega_k^{\star} |\mu_k - \mu'_k(s)| \tag{\Cref{ass:bounded}} \\
        & \le \frac{D}{t} \sum_{s = \widetilde{T}}^t \| \bm\mu - \bm{\mu'}(s) \|_{\infty} \\
        & \le \frac{D}{t} \sum_{s = \widetilde{T}}^t \| \bm\mu - \bmMuHat(s) \|_{\infty} + \| \bmMuHat(s) - \bm{\mu'}(s) \|_{\infty} \\
        & \le \frac{2D \sqrt{2\sigma^2 f(t)}}{t} \sum_{s = \widetilde{T}}^t \max_{k \in [K]} \frac{1}{\sqrt{N_k(s)}} \tag{\Cref{lemma:good-event}, \Cref{ass:subgauss}, $\bm\mu'(s) \in C_s$} \\
        & \le \frac{2D \sqrt{2\sigma^2 f(t)}}{t} \sum_{s = \widetilde{T}}^t \frac{1}{\sqrt{\sqrt{s+K^2}-2K}} \tag{\Cref{lemma:tracking}} \\
        & \le  \frac{2D \sqrt{2\sigma^2 f(t)}}{t} \sqrt{t \sum_{s=\widetilde{T}}^t \frac{1}{\sqrt{s+K^2}-2K}} \tag{Concavity of $\sqrt{\cdot}$} \\
        & \le  \frac{2D \sqrt{2\sigma^2 f(t)}}{t} \sqrt{8t^{3/2}+8Kt\ln(t)}, \tag{Integral test and algebraic manipulations}
    \end{align*}
    which concludes the proof.
\end{proof}

We are now ready to prove \Cref{theo:stas-finite-confidence}.

\begin{proof}[Proof of \Cref{theo:stas-finite-confidence}]
        Let $t \ge K^2$. Then, for $t \ge 10K^4 + T_0(\delta)$, by \Cref{lemma:eq-learning-stas}, we have $\Ecal_t \subseteq \{ \tau_\delta \le t \}$. \Cref{lemma:expectation} and \Cref{lemma:good-event} then conclude the proof.
        To this end, it is sufficient to note that $T^{\star}(\bm\mu)\sum_{i=1}^5 h(t) t + \sqrt{t} + 1\le g(t)$.\footnote{In the relevant regime where $D, L, \sigma, T^{\star}(\bm\mu) \ge 1$.} Here, we followed the same algebraic steps that we presented for the proof of \Cref{theo:proj-tas-finite-confidence}.
\end{proof}

\section{An Explicit Bound}\label{app:explicit-bound}

One concern with the implicit definition of $T_0(\delta)$ in \Cref{theo:proj-tas-finite-confidence,theo:stas-finite-confidence} is whether they yield practically useful rates. In this appendix, we demonstrate that they do. By carefully bounding $T_0(\delta)$, we show that it can be upper-bounded by clean and explicit expression of the form $$\mathcal{O}(T^{\star}(\bm\mu)\log(1/\delta)+T^{\star}(\bm\mu) K \log \log(1/\delta)).$$
These bounds are nearly tight: they match the asymptotic lower bound up to polylogarithmic factors and additive problem-dependent constants. Furthermore, they are also useful in understanding the theoretical guarantees of the algorithms in the regime where $T^{\star}(\bm\mu) \to +\infty$ (\ie in a BAI problem, these are the hard instances where the mean of the second best arm approaches the one of the optimal arm).

This section is structured as follows. First, in \Cref{app:exp-tas} we present the result of \tas, then in \Cref{app:exp-stas} we present the result for \stas.

\subsection{An explicit bound for \tas}\label{app:exp-tas}
For \tas, we show the following upper-bound on $T_0(\delta)$.

\begin{proposition}\label{proposition:explicit}
    Consider any $\eta_1 \in (0, 1/2)$, $\eta_2 \in (0, 1/4)$ and let ${A}_1(\eta_1)$ and ${A}_2(\eta_2)$ be defined as follows:
    \begin{align*}
        A_1(\eta_1) \coloneqq \frac{66 \sigma D L K^2 \log(K)T^{\star}(\bm\mu)}{\eta_1} \quad \quad A_2(\eta_2) \coloneqq \frac{16 \sigma D \sqrt{K} T^{\star}(\bm\mu)}{\eta_2}.
    \end{align*}
    Furthermore, consider any $\alpha, \gamma \in (0,1)$ such that $\alpha + \gamma < 1$ and let $\tilde{A}_1(\eta, \alpha)$ and $\tilde{A}_2(\eta, \gamma)$ be defined as follows:
    \begin{align*}
        \tilde{A}_1(\eta_1, \alpha) \coloneqq A_1(\eta_1) \left( \frac{(0.5+\eta_1)A_1(\eta_1)}{\alpha} \right)^{\frac{0.5+\eta_1}{0.5-\eta_1}} \quad \tilde{A}_2(\eta_2,\gamma) \coloneqq A_2(\eta_2) \left(   \frac{(0.75+\eta_2) A_2(\eta_2)}{\gamma} \right)^{\frac{0.75+\eta_2}{0.25-\eta_2}}
    \end{align*}
    Then, it holds that:
    \begin{align}\label{eq:t0-ub}
        T_0(\delta) \le \frac{T^{\star}(\bm\mu) \log(1/\delta)+T^{\star}(\bm\mu)K\log(\log(1/\delta)+1)+\tilde{A}_1(\eta_1, \alpha)+\tilde{A}_2(\eta_2, \gamma)}{1-\alpha-\gamma}.
    \end{align}
\end{proposition}

Before proving the proposition, we comment on the result.

\Cref{proposition:explicit} provides an upper bound on $T_0(\delta)$ that holds for all $\eta_1, \eta_2,\alpha, \gamma > 0$ such that $\eta_1 < 1/2$, $\eta_2 < 1/4$ and $\alpha+\gamma < 1$. Hence, the tightest bound is achieved while minimizing \Cref{eq:t0-ub} over this domain. Some comments are in order.

\paragraph{Asymptotic Regime of $\delta \to 0$}First, whenever $\delta \to 0$, one can pick any valid $\eta_1, \eta_2$ together with $\alpha$ and $\gamma$ that goes progressively to $0$, yielding to $T_0(\delta) \approx T^{\star}(\bm\mu) \log(1/\delta) + T^{\star}(\bm\mu)K\log(\log(1/\delta)+1)$. This shows how our result retrieves asymptotic optimality together with the dependency on minor order terms of $\delta$. 

\paragraph{Asymptotic Regime of $T^{\star}(\bm\mu) \to +\infty$} Second, by picking $\alpha=\gamma=\frac{1}{4}$, we can easily evaluate the moderate regime of $\delta$ in difficult instances where $T^{\star}(\bm\mu) \to \infty$ (e.g., the case of best-arm identification where the gap between optimal arm and the second best one tends to 0). In this case, one obtains a rate of the form $\mathcal{O}\left( A_1(\eta_1)^{1+\frac{0.5+\eta_1}{0.5-\eta_1}} + A_2(\eta_2)^{1+\frac{0.75+\eta_2}{0.25-\eta_2}} \right)$. To further understand the dependencies in relevant quantities, let us first analyze the first term, that is:
\begin{align*}
    \inf_{\eta \in (0, \frac{1}{2})}A_1(\eta)^{1+\frac{0.5+\eta}{0.5-\eta}} & = \inf_{\eta \in (0, \frac{1}{2})} A_1(\eta)^{\frac{1}{0.5-\eta}} \\ & = \inf_{\eta \in (0, \frac{1}{2})} \left( \frac{66\sigma DLK^2 \log(K) T^{\star}(\bm\mu)}{\eta} \right)^{\frac{1}{0.5-\eta}} \\ & \coloneqq \inf_{\eta \in (0, \frac{1}{2})} \left( \frac{A_1}{\eta} \right)^{\frac{1}{0.5-\eta}},
\end{align*}
where in the last step we introduced $A_1 \coloneqq 66\sigma DLK^2 \log(K) T^{\star}(\bm\mu)$. Let $L_1 = \log(4A_1)$. Then, for sufficiently large $T^{\star}(\bm\mu)$, we have that $\frac{1}{4L_1} \in (0, \frac{1}{2})$. Hence, we obtain that:
\begin{align*}
    \inf_{\eta \in (0, \frac{1}{2})} \left( \frac{A_1}{\eta} \right)^{\frac{1}{0.5-\eta}} & \le \left( 4 A_1 L_1 \right)^{\frac{1}{0.5- \frac{1}{4L_1}}} \\ & = \left( 4 A_1 L_1 \right)^{\frac{4 L_1}{2L_1 - 1}} \\ & = (4A_1 L_1)^{2} (4A_1L_1)^{\frac{2}{2L_1-1}} \\&  = (4A_1 L_1)^{2} \exp \left( \frac{2}{2L_1 -1}\log(4A_1L_1) \right) \\ & = (4A_1 L_1)^2 \exp \left( \frac{4 L_1 - 2 }{2L_1 -1} \right) \\ & \in  \mathcal{O}\left( A_1^2 \log(A_1)^2 \right),
\end{align*}
where, in the forth step, we have used that for $T^{\star}(\bm\mu)$ sufficiently large $4 L_1 - 2\ge 2L_1 + 2 \log(L_1)$. 

Following similar reasoning for the term with $A_2$, we have that:
\begin{align*}
\inf_{\eta \in (0, \frac{1}{4})}A_2(\eta)^{1+\frac{0.75+\eta}{0.25-\eta}} \le \mathcal{O} \left( A_2^4 \log(A_2)^4\right),
\end{align*}
where $A_2 \coloneqq 16 \sigma D \sqrt{K} T^{\star}(\bm\mu)$.

Hence, we have that:
\begin{align*}
    \inf_{\eta_1 \in (0, 1/2), \eta_2 \in (0, 1/4)} A_1(\eta_1)^{1+\frac{0.5+\eta_1}{0.5-\eta_1}} + A_2(\eta_2)^{1+\frac{0.75+\eta_2}{0.25-\eta_2}} \le 2 \max\left\{ A_1^2 \log(A_1)^2, A_2^4 \log(A_2)^4  \right\}
\end{align*}
For $T^{\star}(\bm\mu) \to \infty$, $A_2^4 \log(A_2)^4$ is the dominant term, thus providing the relevant dependencies in term of $D, \sigma, L$, $K$ and $T^{\star}(\bm\mu)$.

We now conclude this section with a proof of \Cref{proposition:explicit}.

\begin{proof}[Proof of \Cref{proposition:explicit}]
    From \Cref{theo:proj-tas-finite-confidence}, we have that:
    \begin{align*}
        T_0(\delta) = \inf\left\{t \in \Naturals: \beta_{t,\delta} \le (t-\sqrt{t} - 1)T^{\star}(\bm\mu)^{-1} - g(t)  \right\},
    \end{align*}
    where $g(t)$ is given by:
    \begin{align*}
        g(t) = 64\sigma DLK^2\log(K)\sqrt{t \log^2(t)} + 16D\sqrt{K t^{3/2} \log(t)}.
    \end{align*}
    
    Using $\log(t) \le \frac{t^{\eta}}{\eta}$ together with the definition of $\beta_{t,\delta}$, \ie \Cref{eq:beta-t-delta}, it follows that $T_0(\delta)$ can be upper-bounded by:\footnote{Here, for simplicity, we incorporated the $\sqrt{t}+1$ term and the $K \log\log(t)$ component of $\beta_{t,\delta}$ within the $\sqrt{t \log(t)}$ term.}
    
    \begin{align*}
        \inf\left\{t \in \Naturals: A_0(\delta) + A_1(\eta_1) t^{0.5+\eta_1}+ A_2(\eta_2)t^{0.75+\eta_2} \le t   \right\},
    \end{align*}
    where, for brevity, we have shortened $T^{\star}(\bm\mu)\log(1/\delta) + T^{\star}(\bm\mu)K\log(\log(1/\delta)+1)$ with $A_0(\delta)$. Note that since $\eta_1 < 1/2$ and $\eta_2 < 1/4$, the upper bound is still well-defined and finite.

    Next, by applying Young's inequality (\Cref{lemma:youg}), we can further upper-bound this expression as follows:\footnote{For the $A_1$ term, apply \Cref{lemma:youg} with $a=t^{0.5+\eta}, b=A_1(\eta), p=0.5+\eta, q=1-0.5-\eta, \epsilon=\alpha/(0.5+\eta)$. Similarly, for the $A_2$ term, use $a=t^{0.75+\eta}, b = A_2(\eta), p=0.75+\eta, q=1-0.75-\eta, \epsilon=\gamma/(0.75+\eta)$.}
    \begin{align*}
        \inf\left\{t \in \Naturals: A_0(\delta) + \alpha t + \tilde{A}_1(\eta_1, \alpha) + \gamma t + \tilde{A}_2(\eta_2, \gamma) \le t   \right\},
    \end{align*}
    Solving for $t$ yields the desired result.    
\end{proof}

\subsection{An explicit bound for \stas}\label{app:exp-stas}

Following the same reasoning that we presented above, it is possible to derive the following explicit expression of $T_0(\delta)$ for \stas.

\begin{proposition}\label{prop:explicit-stas}
    Consider any $\eta_1 \in (0, 1/2)$, $\eta_2 \in (0, 1/4)$ and let ${A}_1(\eta_1)$ and ${A}_2(\eta_2)$ be defined as follows:
    \begin{align*}
        A_1(\eta_1) \coloneqq \frac{80\sigma DLK^2\log(K)\sqrt{t \log^2(t)}}{\eta_1} \quad \quad A_2(\eta_2) \coloneqq \frac{32 \sigma D \sqrt{K} T^{\star}(\bm\mu)}{\eta_2}.
    \end{align*}
    Furthermore, consider any $\alpha, \gamma \in (0,1)$ such that $\alpha + \gamma < 1$ and let $\tilde{A}_1(\eta, \alpha)$ and $\tilde{A}_2(\eta, \gamma)$ be defined as follows:
    \begin{align*}
        \tilde{A}_1(\eta_1, \alpha) \coloneqq A_1(\eta_1) \left( \frac{(0.5+\eta_1)A_1(\eta_1)}{\alpha} \right)^{\frac{0.5+\eta_1}{0.5-\eta_1}} \quad \tilde{A}_2(\eta_2,\gamma) \coloneqq A_2(\eta_2) \left(   \frac{(0.75+\eta_2) A_2(\eta_2)}{\gamma} \right)^{\frac{0.75+\eta_2}{0.25-\eta_2}}
    \end{align*}
    Then, it holds that:
    \begin{align}\label{eq:t0-ub-2}
        T_0(\delta) \le \frac{T^{\star}(\bm\mu)T_{\bm\mu} + T^{\star}(\bm\mu) \log(1/\delta)+T^{\star}(\bm\mu)K\log(\log(1/\delta)+1)+\tilde{A}_1(\eta_1, \alpha)+\tilde{A}_2(\eta_2, \gamma)}{1-\alpha-\gamma}.
    \end{align}
\end{proposition}
\begin{proof}
    The proof is identical to the one of \Cref{proposition:explicit}. The main difference is only in the presence of $T^{\star}(\bm\mu)T_{\bm\mu}$ that is due to the additional complexity that affects \stas.
\end{proof}

\Cref{prop:explicit-stas} provides the explicit expression of $T_0(\delta)$ compared to its implicit version that is presented in \Cref{theo:stas-finite-confidence}. Here, comments that are analogous to those that we presented for \Cref{proposition:explicit} hold.

\section{Auxiliary Tools}\label{app:tech}

\subsection{Technical Tools}

\begin{lemma}[Young's Inequality]\label{lemma:youg}
    Let $a \ge 0$, $b \ge 0$ and consider integers $p, q > 1$ such that $\frac{1}{p}+\frac{1}{q} = 1$. Furthermore, let $\epsilon > 0$. Then, it holds that:
    \begin{align*}
        ab \le \frac{\epsilon a^p}{p} + \frac{b^q}{q \epsilon^{q/p}}.
    \end{align*}
\end{lemma}

\subsection{Cumulative Tracking}
The following lemma summarizes the main properties of the C-Tracking procedure. 

\begin{lemma}[C-Tracking]\label{lemma:tracking}
    Let $\{ \bm\omega(t) \}$ be an arbitrary sequence of elements that belongs to a $K$-dimensional simplex. Consider the C-Tracking applied on a sequence $\{ {\bm\omega}(t) \}_{t}$ and let us denote by $\tilde{\bm\omega}(t)$  the $l_{\infty}$ projection of $\bm\omega(t)$ onto $\Delta_K^{\epsilon_t} = \Delta_K \cap [\epsilon_t, 1]$ and $\epsilon_t = (K^2 +t)^{-1/2}/2$.
    For all $t \in \mathbb{N}$, it holds that:
    \begin{align}
        & N_k(t) \ge \sqrt{t+K^2}-2K \label{eq:track-one}\\ 
        & - K \ln(K) \sqrt{t+K^2}\le N_k(t) - \sum_{s=1}^t \omega_k(s) \le  K \sqrt{t+K^2} \label{eq:track-two} \\
        & \sum_{s=1}^t \sum_{k \in [K]} \frac{\omega_k(s)}{\sqrt{N_k(t)}} \le K \ln(K) + 4 \sqrt{Kt} + K^2\sqrt{t+K^2} \label{eq:track-three}
    \end{align} 
\end{lemma}
\begin{proof}
    \Cref{eq:track-one} is due to Lemma 7 in \cite{garivier2016optimal} and it is due to the forced exploration of the C-Tracking procedure (\ie the projection onto $\Delta_{K}^{\epsilon_t}$).

    Now, from Theorem 6 in \cite{degenne2020structure}, we have that:
    \begin{align}\label{eq:track-proof}
        -\ln(K)\le N_k(t) - \sum_{s=1}^t \tilde{\omega}_k(s) \le 1.
    \end{align}
    However, from Lemma 7 in \cite{garivier2016optimal}, we also have that:
    \begin{align}\label{eq:track-proof-2}
        \max_{k \in [K]} \Big|\sum_{s=1}^t \omega_k(s) - \tilde{\omega}_k(s) \Big| \le K \sqrt{t+K^2} .
    \end{align}
    Combining this result with \Cref{eq:track-proof} leads to \Cref{eq:track-two}.

    Finally, from Lemma 6 in \cite{degenne2020structure}, we have that:
    \begin{align*}
        & \sum_{s=1}^t \sum_{k \in [K]} \frac{\tilde{\omega}_k(s)}{\sqrt{N_k(t)}} \le K \ln(K) + 4 \sqrt{Kt}
    \end{align*}
    Combining these results with \Cref{eq:track-proof-2} leads to \Cref{eq:track-three}, thus concluding the proof.
\end{proof}

\subsection{Properties of Canonical Exponential Families}

The following lemma reports standard properties of one-dimensional canonical exponential families.  

\begin{lemma}[KL Difference in Exponentialy Families]\label{lemma:canonical-exp-family-diff}
    For three distributions in a canonical exponential family with means $a,b,c$, it holds that:
    \begin{align}
        & d(a,b) = d(a,c) + d(c,b) + (\nu_b - \nu_c) (c-a) \label{eq:canonical-1}\\
        & d(c,b) - d(a,b) \le (\nu_c - \nu_b)(c-a) \label{eq:canonical-2}
    \end{align}
    where $\nu_{(\cdot)}$ denotes the natural parameter of the distribution with mean $(\cdot)$.  
\end{lemma}
\begin{proof}
    For a proof, see, \eg Lemma E.6 in \cite{poiani2024best}.
\end{proof}

\subsection{Concentration Results}

\begin{lemma}[Good Event]\label{lemma:good-event}
    Consider $\{ \Ecal_t \}_t$ such that
    \begin{align*}
        \Ecal_t = \left\{ \forall s \in \left[\lceil \sqrt{t} \rceil, t\right], \sum_{k \in[K]} N_k(s) d(\hat{\mu}_k(s), \mu_k) \le 8K \log(s) \right\}
    \end{align*}
    It holds that $\sum_{t=3}^{+\infty} \mathbb{P}_{\bm\mu}(\Ecal_t^c) \le 2eK$.
\end{lemma}
\begin{proof}
    The statement is a direct corollary of standard concentration arguments (see Lemma 6 in \citet{degenne2019non}).
\end{proof}

\subsection{$\delta$-Correctness}

For simplicity of exposition, we consider the following choice of the threshold $\beta_{t,\delta}$:
\begin{align}\label{eq:beta-t-delta}
    \beta_{t,\delta} = \log\left( \frac{1}{\delta} \right) + K \log \left( 4 \log\left( \frac{1}{\delta}\right) +1 \right) + 6K \log(\log(t)+3).
\end{align}
This threshold has been shown to yield $\delta$-correct algorithms for Gaussian distributions \cite{menard2019gradient}. At a cost of more involved expression, one can adopt the threshold proposed in \cite{kaufmann2021mixture} to analyze the stopping time in generic canonical exponential families. 

We now prove for completeness that this choice of $\beta_{t,\delta}$ combined with the stopping and recommendation rules leads to a $\delta$-correct algorithm for any sampling rule. Note that the above result holds for both the cases when $i^{\star}(\bm\mu)$ is both single and multiple-valued.

\begin{lemma}[Correctness]\label{lemma:correctness}
    For any sampling rule and $\bm\mu \in \Mcal$, it holds that $\Prob_{\bm\mu}(\hat{\imath}_{\tau_\delta} \notin i^{\star}(\bm\mu) ) \le \delta$.
\end{lemma}
\begin{proof}
    With probabilistic arguments, we have that:
    \begin{align*}
        \Prob_{\bm\mu}(\hat{\imath}_{\tau_\delta} \notin i^{\star}(\bm\mu)) & \le \Prob_{\bm\mu}\left( \exists t \in \Naturals, j \notin i^{\star}(\bm\mu): \inf_{\bm\lambda \in \lnot j} \sum_{k \in [K]} N_k(t) d(\hat{\mu}_k(t), \lambda_k) \ge \beta_{t,\delta}\right) \\ 
        & \le \Prob_{\bm\mu}\left( \exists t \in \Naturals: \sum_{k \in [K]} N_k(t) d(\hat{\mu}_k(t), \mu_k) \ge \beta_{t,\delta} \right) \\
        & \le \delta,
    \end{align*}
    where the first step follows from the definition of the stopping and recommendation rules, the second one from the fact that $\bm\mu \in \lnot j$ for all $j \notin i^{\star}(\bm\mu)$, and the third one from Proposition 1 in \cite{menard2019gradient}.
\end{proof}

\end{document}